\documentclass{article}

\usepackage{microtype}
\usepackage{graphicx}
\usepackage{subcaption}
\usepackage{booktabs} 

\usepackage[pagebackref]{hyperref}

\usepackage[accepted]{icml2026}

\usepackage{amsmath}
\usepackage{amssymb}
\usepackage{mathtools}
\usepackage{amsthm}

\usepackage{wrapfig}
\usepackage{titletoc}

\usepackage[capitalize,noabbrev]{cleveref}

\usepackage{amsmath,amsfonts,bm}

\def\1{\bm{1}}

\DeclareMathAlphabet{\mathsfit}{\encodingdefault}{\sfdefault}{m}{sl}
\SetMathAlphabet{\mathsfit}{bold}{\encodingdefault}{\sfdefault}{bx}{n}

\DeclareMathOperator{\sign}{sign}

\newcommand{\trinorm}[1]{{\left\vert\kern-0.25ex\left\vert\kern-0.25ex\left\vert #1 
   \right\vert\kern-0.25ex\right\vert\kern-0.25ex\right\vert}}

\newcommand{\EMA}{{\rm EMA}}

\newcommand{\bg}{\boldsymbol{g}}

\newcommand{\bu}{\boldsymbol{u}}
\newcommand{\bv}{\boldsymbol{v}}

\newcommand{\btheta}{\boldsymbol{\theta}}

\newcommand{\cL}{\mathcal{L}}

\newcommand{\cO}{\mathcal{O}}

\newcommand{\cQ}{\mathcal{Q}}

\newcommand{\bbE}{\mathbb{E}}

\newcommand{\bbN}{\mathbb{N}}

\newcommand{\bbP}{\mathbb{P}}

\newcommand{\bbR}{\mathbb{R}}

\newcommand{\bzero}{\mathbf{0}}

\newcommand{\pll}{\kern 0.56em/\kern -0.8em /\kern 0.56em}

\newcommand{\norm}[1]{\ensuremath{\left\| #1 \right\|}}

\newcommand{\<}{\left\langle}
\renewcommand{\>}{\right\rangle}

\theoremstyle{plain}
\newtheorem{theorem}{Theorem}[section]
\newtheorem{proposition}[theorem]{Proposition}
\newtheorem{lemma}[theorem]{Lemma}
\newtheorem{example}[theorem]{Example}

\theoremstyle{definition}

\newtheorem{assumption}[theorem]{Assumption}
\theoremstyle{remark}

\usepackage{enumitem}
\definecolor{thistle}{rgb}{0.85, 0.75, 0.85}
\definecolor{ork}{rgb}{0.89, 0.79, 0.65}
\renewcommand{\geq}{\geqslant}
\renewcommand{\leq}{\leqslant}

\newcommand{\ema}{\texttt{EMA}}
\newcommand{\RR}{\mathbb{R}}

\usepackage[textsize=tiny]{todonotes}

\icmltitlerunning{GradPower: Powering Gradients for Faster Language Model Pre-Training}

\begin{document}

\twocolumn[
  \icmltitle{GradPower: Powering Gradients for Faster Language Model Pre-Training}



\icmlsetsymbol{equal}{*}
\icmlsetsymbol{correspond}{$\dag$}

\begin{icmlauthorlist}
\icmlauthor{Jinbo Wang}{equal,sms}
\icmlauthor{Mingze Wang}{equal,correspond,sms}
\icmlauthor{Jiaqi Zhang}{correspond,mt}
\icmlauthor{Wei Wang}{mt}
\\
\icmlauthor{Peng Pei}{mt}
\icmlauthor{Xunliang Cai}{mt}
\icmlauthor{Weinan E}{sms,cmlr,aisi}
\icmlauthor{Lei Wu}{correspond,sms,cmlr,aisi}
\end{icmlauthorlist}

\icmlaffiliation{sms}{School of Mathematical Sciences, Peking University}
\icmlaffiliation{cmlr}{Center for Machine Learning Research, Peking University}
\icmlaffiliation{mt}{Meituan, Beijing}
\icmlaffiliation{aisi}{AI for Science Institute, Beijing, China}

\icmlcorrespondingauthor{Mingze Wang}{mingzewang@stu.pku.edu.cn}
\icmlcorrespondingauthor{Jiaqi Zhang}{zhangjiaqi39@meituan.com}
\icmlcorrespondingauthor{Lei Wu}{leiwu@math.pku.edu.cn}

  \icmlkeywords{Machine Learning, ICML}

  \vskip 0.3in
]



\printAffiliationsAndNotice{\icmlEqualContribution}

\begin{abstract}
We propose {\bf GradPower}, a lightweight gradient-transformation technique for accelerating language model pre-training.  Given a gradient vector $\bg=(g_i)_i$,   GradPower  first applies the  elementwise \texttt{sign-power}  transformation
$\varphi_p(\bg)=\left(\sign(g_i)|g_i|^p\right)_{i}$
for a fixed $p>0$, and then feeds the transformed gradient  into a base optimizer. Notably, GradPower requires only a {\bf single-line code change} and no modifications to the base optimizer’s internal logic, including the hyperparameters. 
When applied to AdamW (termed {\bf AdamWPower}), GradPower consistently achieves lower terminal loss across diverse architectures (LLaMA, Qwen2MoE), parameter scales (66M to 2B), datasets (C4, OpenWebText), and learning-rate schedules (cosine, warmup-stable-decay). 
The most pronounced gains are observed when training modern mixture-of-experts models with warmup-stable-decay schedules. GradPower also  integrates seamlessly with other state-of-the-art optimizers, such as Muon, yielding further improvements. Finally, we provide  theoretical analyses that reveal the underlying mechanism of GradPower and highlight the influence  of gradient noise.
\end{abstract}



\section{Introduction}\label{sec: introduction}
Large language models (LLMs) have revolutionized modern artificial intelligence~\citep{brown2020language,achiam2023gpt,liu2024deepseek}.
However, pre-training LLMs is computationally intensive due to the massive scale of model size and training data. Improving the pre-training efficiency has thus become a primary objective in the continued  scaling of LLMs. Among the factors affecting efficiency, the choice of optimizer is critical. In practice, the Adam optimizer~\citep{kingma2014adam,loshchilov2017decoupled} has emerged as the de facto choice in most LLM pre-training pipelines, owing to its adaptive learning rate features~\citep{zhang2024transformers,kunstner2024heavy}.

To further accelerate AdamW, several approaches have been proposed to refine or simplify its moment estimation~\citep{xie2022adan,pagliardini2024ademamix,chen2024symbolic,liu2025focus,zhang2024adam}.  Other strategies  modify the update rule directly, such as  direction correction~\citep{wang2024improving,liang2024cautious}, incorporating curvature information~\citep{liu2023sophia,wang2025sharpness}, or applying matrix-based preconditioning~\citep{jordan2024muon,liu2025muon}.  
While these methods often deliver tangible gains, they typically require substantial modifications to the existing training pipeline  and careful extra hyperparameter tuning, which hinders their practical adoption.

\begin{figure*}[!htb]
    \centering
    \includegraphics[width=0.28\linewidth]{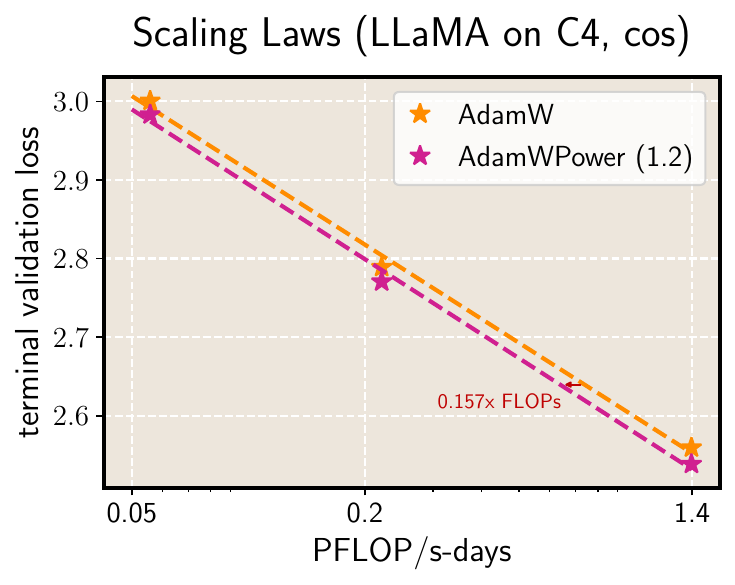}
    \includegraphics[width=0.28\linewidth]{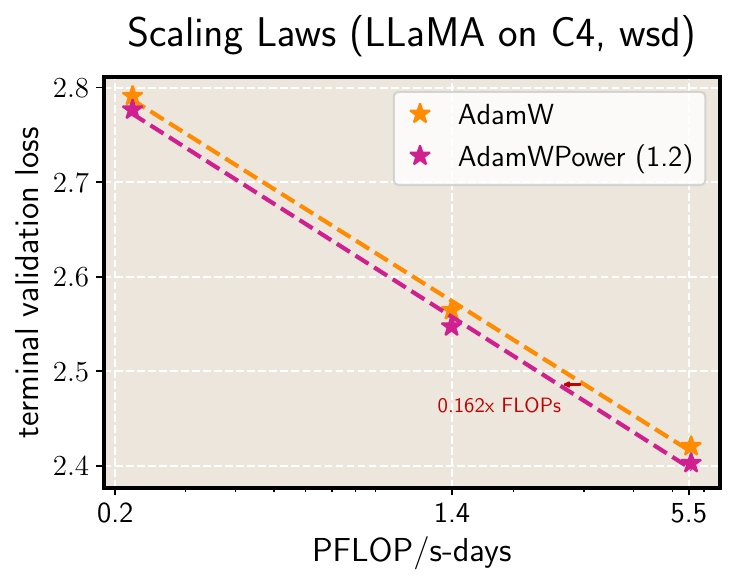}
    \includegraphics[width=0.28\linewidth]{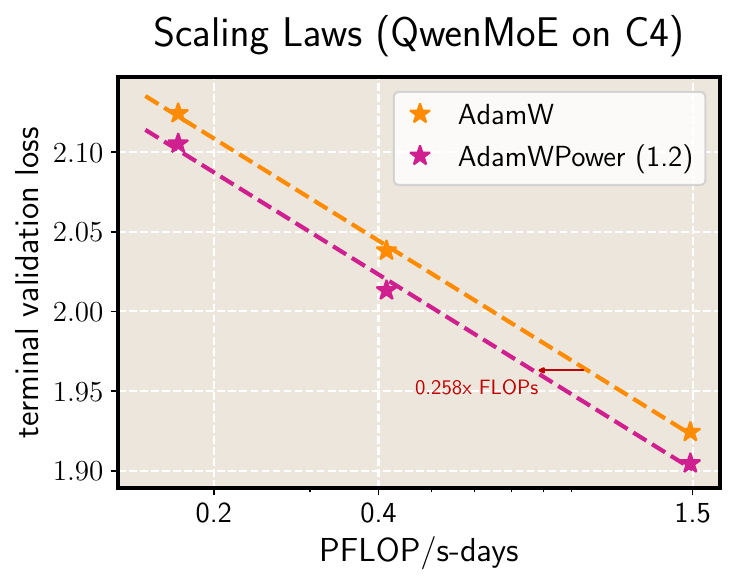}
\caption{\small
Scaling-law comparison of AdamWPower and AdamW on the C4 dataset for dense LLaMA models and mixture-of-experts Qwen2MoE models.}
    \label{fig: scaling law on c4}
    \vspace*{-.5em}
\end{figure*}

In contrast to these intrusive modifications, we instead propose a lightweight, plug-in approach by revisiting the core component of optimization: the {\em gradient itself}.  We apply a simple elementwise transformation to the gradient vector -- enhancing its informativeness while leaving the base optimizer entirely unchanged. This design preserves compatibility with existing training pipelines and avoids additional tuning burden. Specifically, {\bf our contributions} are  as follows:

\begin{itemize}[leftmargin=1em]
    \item \textbf{Our approach.}  
    We propose GradPower, a simple but effective approach for boosting the convergence of general gradient-based optimizers.
    Specifically, given a raw gradient  $\bg=(g_i)_{i}\in\mathbb{R}^d$ and a fixed $p>0$, we define the {\bf powered  gradient} as
    \begin{equation}
    \label{eqn: grad-power}
    \begin{split}
        &\varphi_p(\bg) := |\bg|^p \sign(\bg) \\=&
        \left( |g_1|^p\sign(g_1)
        ,\dots, |g_d|^p\sign(g_d) \right)^\top.
    \end{split}
    \end{equation}
    GradPower applies this powered gradient to the base optimizer, preserving its original structure.

\item \textbf{Empirical performance.} We first evaluate the effectiveness of GradPower by integrating it into AdamW, termed \textbf{AdamWPower}. We test its performance across a broad LLM pre-training landscape: {\bf dense} models (LLaMA~\citep{touvron2023llama}) and {\bf mixture-of-experts} models (Qwen2MoE~\citep{yang2024qwen2technicalreport}), ranging from {\bf 66M} to {\bf 2B} parameters, using the C4 and OpenWebText corpora, and under both cosine-decay (\texttt{cos}) and warmup-stable-decay (\texttt{wsd}) learning-rate schedules.
Across all settings, AdamWPower consistently achieves lower terminal loss and exhibits more favorable scaling laws compared to vanilla AdamW (see Figure~\ref{fig: scaling law on c4}), demonstrating its potential for improved scalability to larger models. Notably, the performance gains are most significant for modern MoE architectures and \texttt{wsd} schedules.

Furthermore, we show that GradPower can be also seamlessly integrated with other state-of-the-art optimizers, such as Muon~\citep{jordan2024muon, liu2025muon} and Blockwise LR~\citep{wang2025sharpness}, yielding additional performance improvements.

    \item \textbf{Theoretical analysis.} 
1) Recent analyses suggest that steady progress along flat ``river-like'' directions is crucial for reducing loss in LLM pre-training~\citep{wang2024improving,wen2024understanding}. We show that AdamWPower amplifies these directions, thereby accelerating optimization.
2) Moreover, for general \emph{smooth non-convex objectives}, we prove that augmenting adaptive optimizers (e.g., AdaGrad) with GradPower strictly accelerates their convergence in both low- and high-noise regimes, supporting the intuitions developed in Section~\ref{sec: motivation}.
\end{itemize}


{\bf Notations.}
For  $\{a_s\}_{s=1}^{\infty}$, its $\beta$-exponential moving average at time $t$ is denoted as $\ema_\beta(\{a_s\}_{1}^t):=(\sum_{s=1}^t\beta^{t-s}a_s)/(\sum_{s=1}^t\beta^{t-s})$.
For $g\in\bbR$ and $p\in\bbR_+$, we denote $g^p:=|g|^p\sign(g)$; for a vector $\bg$, the notation $\bg^p$ denotes element-wise application of this transformation.
For simplicity, we use {\em a.s.} to denote ``almost surely'', and use {\em w.r.t.} to denote ``with respect to''.
We use standard big-O notations $\cO(\cdot),\Omega(\cdot),\Theta(\cdot)$ to hide problem-independent constants, and use $o(\cdot)$ to denote the infinitesimal.
Let $\norm{\cdot}_q$ denote the $\ell_q$ norm for vectors for a $q\geq 1$.
We denote $[n]=\{1,\cdots,n\}$ for an integer $n\in\bbN_+$.



\section{Related Works}\label{sec: related works}

{\bf Optimizer design in LLM pre-training.} 
In LLM pre-training, Adam~\citep{kingma2014adam} has become the de facto optimizer. 
Recent efforts to improve its efficiency focus on two aspects: accelerating convergence and reducing memory usage. Techniques for
{\em accelerating convergence}  include introducing curvature information~\citep{liu2023sophia,wang2024improving,wang2025sharpness}, mixing momentum~\citep{xie2022adan,pagliardini2024ademamix}, variance reduction~\citep{yuan2024mars}, cautious update~\citep{liang2024cautious}, and applying matrix-based preconditioners~\citep{jordan2024muon,vyas2024soap}.  
{\em Memory-efficient}  techniques  include reducing the moments usage in Adam~\citep{zhang2024adam},  sign-based updates~\citep{chen2024symbolic,liu2025focus}, low precision optimizer states~\citep{dettmers20218,li2023memory}, low-rank approximation~\citep{hu2022lora, zhao2024galore,chen2024fira}, and structured learning rates~\citep{zhu2024apollo}.
Among these, Muon~\citep{jordan2024muon} stands out for  improving both convergence  and memory usage and showing strong scalability~\citep{liu2025muon}.
{\em In contrast}, our  method, GradPower, improves training efficiency without altering the base optimizer’s internal updates. Notably, GradPower is orthogonal and complementary to the methods above: it can serve as a lightweight plug-in that further enhances existing optimizers. 

{\bf The fast-slow dynamics} in neural network training.
Recent works~\citep{wu2018sgd,Jastrzebski2020The,cohen2021gradient,cohen2022adaptive} show that neural network training typically occurs at the so-called Edge of Stability (EoS) stage. This regime is characterized by the optimizer exhibiting oscillatory behavior along sharp directions without divergence, while steadily progressing along {\em flat directions}, leading to loss reduction. Several studies \citep{wen2024understanding,song2024does,cohen2024understanding,wang2024improving,wang2025sharpness,zhu2026accelerating} have emphasized the importance of the slow dynamics along flat directions (referred to as stable direction in \citet{wang2024improving}, river directions by \citet{wen2024understanding} and bulk directions by \citet{song2024does}), in {\em reducing total loss}. Moreover, ~\citet{wen2024understanding} further showed that this picture is crucial for understanding the behavior of LLM pre-training.
In addition, Fig.3 in~\citep{wen2024understanding} and Fig.8 \citep{song2024does} suggest that, the optimizer’s trajectory within flat directions tends to {\em remain aligned} for a period of time.

{\bf The terminology of flat directions.} In classical optimization theory, flat directions refer to Hessian eigenvectors associated with small eigenvalues. However, {\em our flat directions correspond to small stochastic gradient directions}. 
Our usage of flat directions is approximate and follows two lines of prior works:
{\em (i) The anisotropy of gradient noise closely reflects the Hessian’s curvature structure.}
A line of works~\citep{zhu2019anisotropic,wu2020noisy,mori2022power,wu2022does,wang2023theoretical} have established that {\em directions with small Hessian curvature exhibit low gradient-noise variance}, while directions with large Hessian curvature exhibit high gradient-noise variance. 
{\em (ii) LLM pre-training typically operates in a noise-dominated regime, where the stochastic gradient is largely governed by the gradient noise}. The pretraining dataset is massive, and each mini-batch constitutes only a tiny fraction of it. As a result, the stochastic gradient is dominated by sampling noise. Therefore, relative differences across coordinates in stochastic gradient $g$ primarily arise from differences in the noise $\xi$, rather than differences in the true gradient $\nabla\cL$ (Note: $\xi=g-\nabla\cL$). Moreover, prior empirical studies in deep-learning settings show that as training progresses, the gradient noise constitutes the dominant component of the stochastic gradient, and the dynamics occurs in the noise-dominated regime~\citep{mccandlish2018empirical,sun2023unleashing}. For LLMs, direct empirical verification is challenging because computing the full-batch gradient is infeasible, but the above intuition suggests that LLM training should also lie in this noise-dominated regime.
{\em Therefore}, our flat directions approximately correspond to small stochastic gradient directions.

\section{The GradPower Approach}\label{sec: motivation}

Let $\bg_t \in \RR^d$ denote the stochastic gradient at step $t$. A gradient-based optimizer can be expressed as
$\btheta_{t+1} = \btheta_t - \eta_t \cQ(\bg_1, \cdots, \bg_t)
$,
where $\eta_t$ is learning rate and $\cQ$ denotes update rule. 


\paragraph{A unified view of preconditioning.}
 In practice, raw gradients may not be sufficiently informative or stable, and thus, it is common to transform the gradients before applying the update rule. This leads to the general form of {\it preconditioned optimizers}:
\begin{equation}\label{eqn: pre-gradient-optimizer}
\btheta_{t+1} = \btheta_t - \eta_t \cQ\left(\varphi(\bg_1), \cdots, \varphi(\bg_t)\right)
\end{equation}
where $\varphi$ denotes a transformation (or preconditioning) applied to each gradient.


To {\em avoid computational overhead}, we restrict attention to {\em elementwise} transformations. For a gradient vector $\bg=(g_i)_i\in\bbR^d$, we consider
$\varphi(\bg) := (\varphi(g_1), \dots, \varphi(g_d))^\top \in \RR^d$,
where $\varphi: \RR \to \RR$ is a scalar nonlinearity applied coordinate-wise. 
The function $\varphi$ is designed to enhance the informativeness of the raw gradient.
The simplest choice is a linear transformation $\varphi(z)=cz$ with $c\in\bbR$.
However, as shown in Appendix~\ref{appendix: proof of case study}, such linear preconditioners may exhibit limited effectiveness when used in modern optimizers for LLM pre-training.
In this work, we explore {\bf nonlinear preconditioning} as an alternative. 


\paragraph{The GradPower family.}
Empirically, we find that a simple power transformation already yields non-trivial improvements in LLM pre-training. Specifically, we define the \texttt{sign-power} transformation $\varphi_p:\bbR\to\bbR$ with exponent $p > 0$ as
\vspace*{-.3em}
\begin{align*}
    \varphi_p(z)=|z|^p{\rm sign}(z).
\end{align*}
The powered gradient is shown in Equation~\eqref{eqn: grad-power}. 
Incorporating this transformation into AdamW yields a new optimizer, {\bf AdamWPower}, described in Algorithm~\ref{alg: adamwpower}. Remarkably, AdamWPower 
{\colorbox{orange!20}{\!introduces only one additional line of code\!}} 
compared to standard AdamW.

In the following sections, we first present empirical evidence demonstrating the effectiveness of AdamWPower, followed by a theoretical analysis that sheds light on its underlying mechanisms.

\begin{algorithm}[!ht]
	\caption{{\bf AdamWPower}}
	\label{alg: adamwpower}


    
        
        

        

        
        
    \begin{algorithmic}[1]
    \STATE \textbf{Given} hyperparameters $\beta_1,\beta_2,\epsilon,\lambda$; Learning rates $\{\eta_t\}_{t=1}^T$;  \colorbox{orange!20}{\!\!power $p\in\bbR_+$\!\!};
    \STATE \textbf{Initialize} $\btheta_0\in\bbR^d$, first momentum $\boldsymbol{m}_0\leftarrow\bzero$, second momentum $\bv_0\leftarrow\bzero$;
    \FOR{$t=1,\ldots,T$}
        \STATE Compute mini-batch gradient $\bg_t$;
        \STATE \colorbox{orange!20}{\!\textbf{GradPower}: compute powered gradient\!}\\
        \colorbox{orange!20}{\!$\bg_t\leftarrow|\bg_t|^p\sign(\bg_t)$ using Eq.~\eqref{eqn: grad-power};\!}
        \STATE $\boldsymbol{m}_t \leftarrow \beta_1\boldsymbol{m}_{t-1}+(1-\beta_1)\bg_t$;
        \STATE $\hat{\boldsymbol{m}}_t\leftarrow \boldsymbol{m}_t/(1-\beta_1^t)$;
        \STATE $\bv_t \leftarrow \beta_2\boldsymbol{v}_{t-1}+(1-\beta_2)\bg_t^2$;
        \STATE $\hat{\boldsymbol{v}}_t\leftarrow \boldsymbol{v}_t/(1-\beta_2^t)$;
        \STATE $\btheta_{t}\leftarrow\btheta_{t-1}-\eta_t\Big(\hat{\boldsymbol{m}}_t/\left(\sqrt{\hat{\bv}_t}+\epsilon\right)+\lambda\btheta_{t-1}\Big)$;
    \ENDFOR
    \STATE \textbf{Output:} Optimized parameters $\btheta_T$.
    \end{algorithmic}
\end{algorithm}



\section{Empirical Evaluation}\label{sec: experiments}


\subsection{Experimental Setup}


We evaluate AdamWPower for the task of LLM pre-training across a range of  model architectures, parameter scales, datasets, and learning rate (LR) schedulers.
The main experimental configurations are summarized below, while additional implementation details are provided in Appendix~\ref{appendix: experiments}.


\begin{itemize}[leftmargin=1em]
    \item {\bf Models.} 
    We consider two widely used LLM architectures: {\bf LLaMA} (dense) models~\citep{touvron2023llama} and {\bf Qwen2MoE} (MoE) models~\citep{yang2024qwen2technicalreport}. We experiment with model sizes ranging {\bf from 66M to 2B} parameters.

    
    \item {\bf Datasets.} 
    We evaluate our methods on the {\bf Colossal Clean Crawled Corpus (C4)} dataset~\citep{raffel2020exploring}\footnote{A large-scale public language dataset, widely used for LLM pre-training such as T5~\citep{raffel2020exploring}, and prior pre-training studies~\citep{zhao2024galore,zhao2024deconstructing}.} and {\bf OpenWebText} dataset~\citep{Gokaslan2019OpenWeb}\footnote{An open-source recreation of the WebText corpus, commonly used in pre-training models such as RoBERTa~\citep{liu2019roberta}, GPT-2, and NanoGPT~\citep{Karpathy2022}.}.
    For pre-training on  C4,  we follow the setup of \citet{wortsman2023small,zhao2024deconstructing}, using a 
     batch size of 512. 
  The total number of training tokens is set to be approximately 20 times the number of model parameters, in accordance with the Chinchilla scaling law~\citep{hoffmann2022training}.

    \item {\bf LR schedulers.}
    We evaluate two popular LR scheduling strategies: (i) \texttt{cos}: a linear warm-up to peak \texttt{lr\_max}, followed by cosine decay to a terminal LR \texttt{lr\_min}.
    (ii) \texttt{wsd} (warmup-stable-decay scheduler)~\citep{zhai2022scaling,hu2024minicpm,hagele2024scaling}:  a linear warm-up LR to peak \texttt{lr\_max}, followed by a stable phase where LR remains at \texttt{lr\_max} (up to 80\% of the total training steps), and then a linear decay to \texttt{lr\_min}. It should be noticed that \texttt{wsd} introduces a non-traditional loss curve: slowly decrease during the stable phase and suddenly drop during the final decay phase. 
\end{itemize}
We further evaluate our method on vision tasks, and report detailed implementation settings in Appendix~\ref{appendix: experiments}.



\begin{wrapfigure}{r}{0.22\textwidth}
    \vspace{-.3cm}
    \centering
    \includegraphics[width=0.21\textwidth]{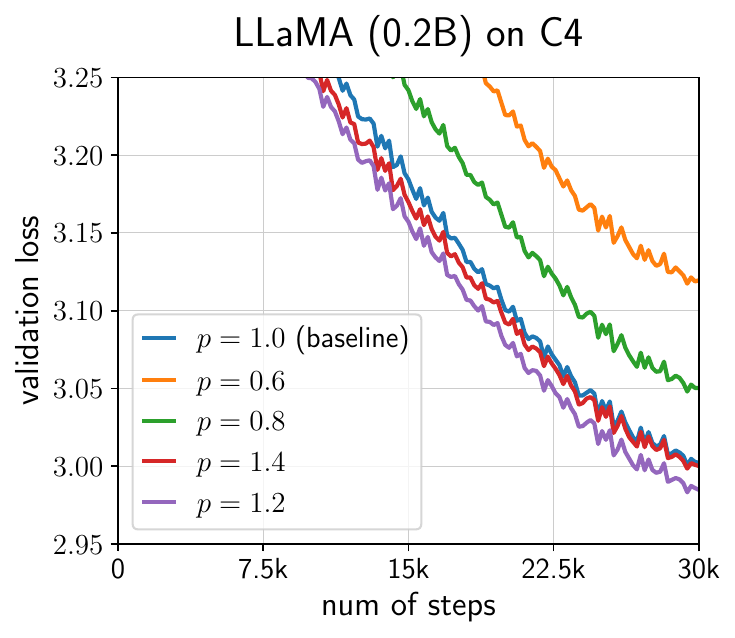}
    \vspace{-.1cm}
    \caption{\small Pre-training LLaMA (0.2B) on C4 using AdamWPower with different power $p$'s. The optimal power is $1.2$.}
    \label{fig: optimal p, 0.2B on C4}
    \vspace{-.2cm}
\end{wrapfigure}
\paragraph{AdamW baselines.} 
We use the standard Adam optimizer (with decoupled weight decay) as the baseline in most experiments (except Section~\ref{subsec: blockwise lr and muon}). 
The baseline is configured with hyperparameters $\beta_1=0.9$, $\beta_2=0.95$, weight decay $\lambda=0.1$, and gradient clipping threshold of $1.0$, following protocols used in LLaMA pre-training~\citep{touvron2023llama}.
For each experiment, we first tune \texttt{lr\_max} over $\{\texttt{1e-4, 2e-4, 3e-4, 6e-4, 1e-3, 1.5e-3}\}$ to be optimal for AdamW, and the baselines are trained using these optimal \texttt{lr\_max}'s.
Details can be found in Appendix~\ref{appendix: experiments}.

\textbf{The tuning of power $p$ and its transferability.}
We only tune the power $p$ in a single small-scale experiment: pre-training LLaMA (0.2B) on C4. As shown in Figure~\ref{fig: optimal p, 0.2B on C4}, the tuned power is $1.2$.
Interestingly, this aligns with the optimal power observed in the high-noise regime of our illustrative toy example (Figure~\ref{fig: toy example, optimal p}). 
We adopt $p=1.2$ as the default in most experiments (except Section~\ref{subsec: experiment: batch size}).
Importantly, the power proves to be  {\bf highly robust}: AdamWPower with $p=1.2$ {\bf consistently outperforms} AdamW and exhibits better scaling laws, across model types, model sizes, datasets, and LR schedulers.

\textbf{Lightweight wall-clock time overhead.} We empirically measured wall-clock time using LLaMA (0.25B) on OpenWebText with batch size 480 and observed that AdamWPower introduces only a 0.4\% overhead per step compared to AdamW, with 0.7565s versus 0.7534s, almost negligible in practice, confirming its lightweight nature.

\subsection{Results on Dense Models}
\label{subsec: dense results}


\begin{figure*}[!htb]
    \centering
    \includegraphics[width=0.24\linewidth]{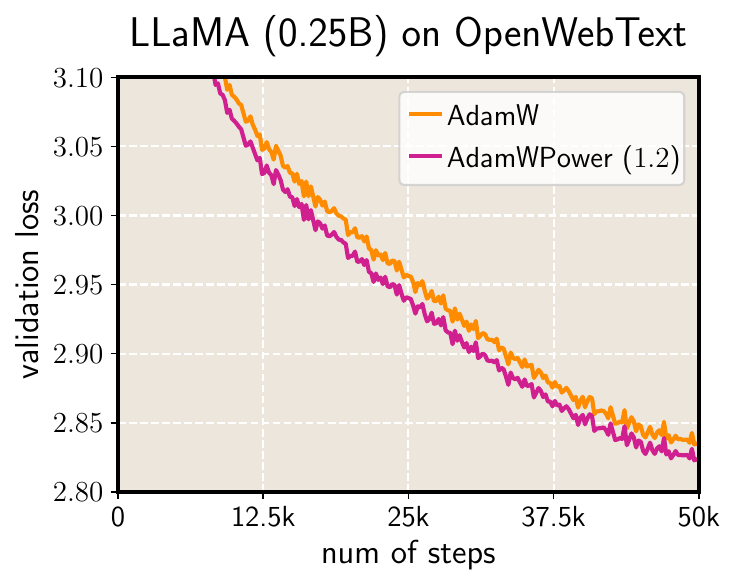}
    \includegraphics[width=0.24\linewidth]{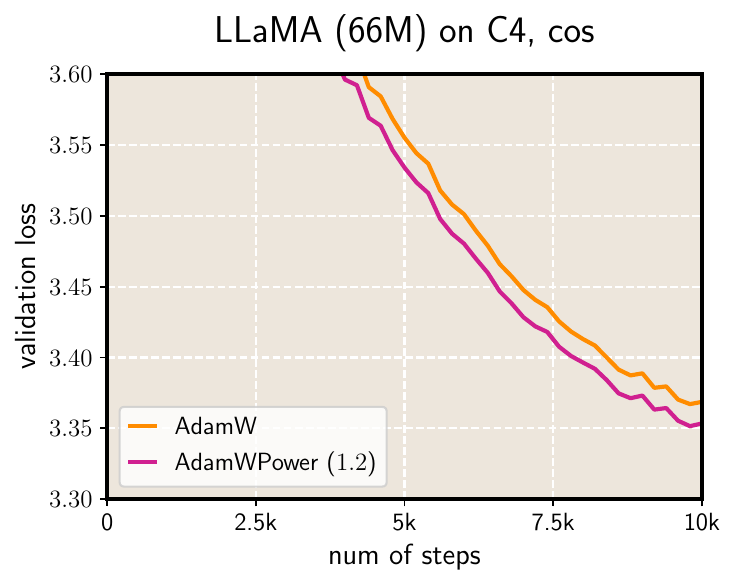}
    \includegraphics[width=0.24\linewidth]{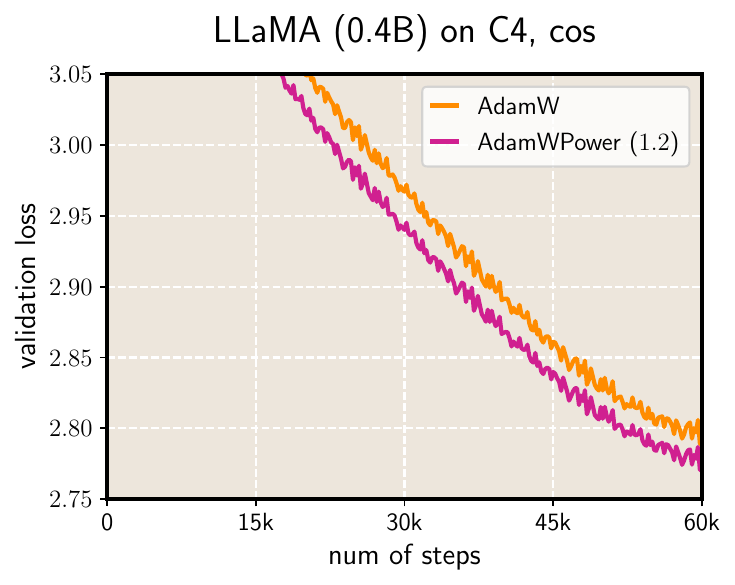}
    \includegraphics[width=0.24\linewidth]{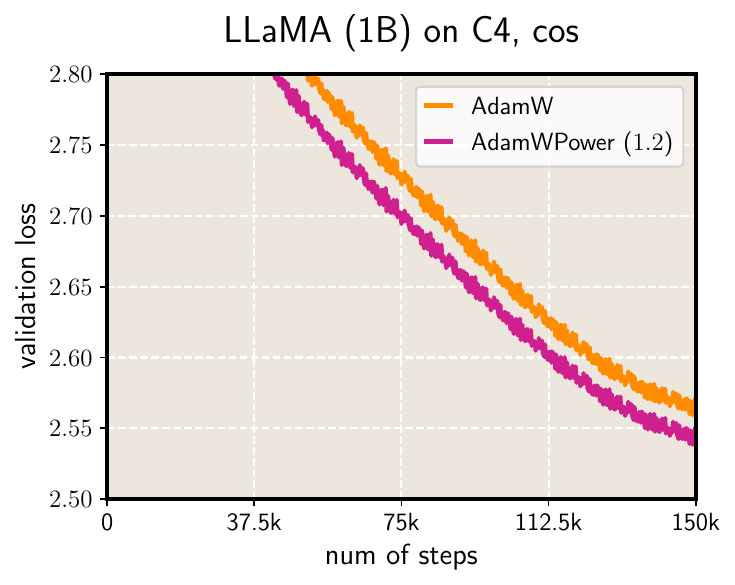}
    \includegraphics[width=0.24\linewidth]{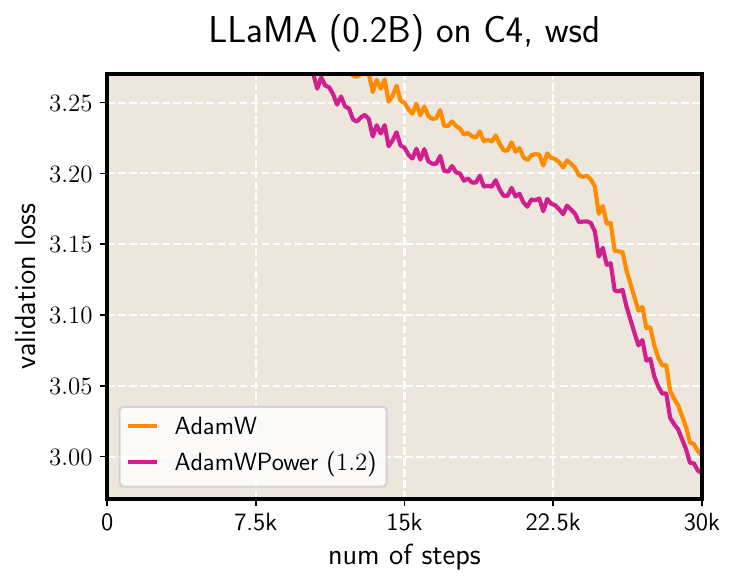}
    \includegraphics[width=0.24\linewidth]{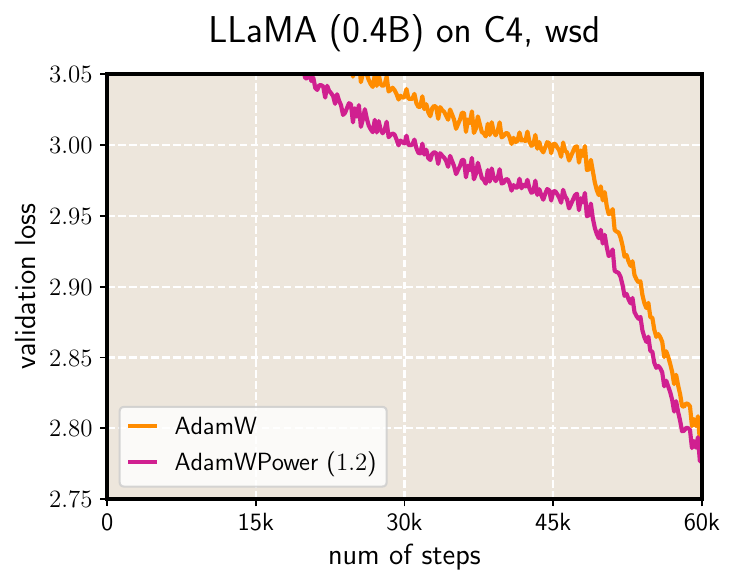}
    \includegraphics[width=0.24\linewidth]{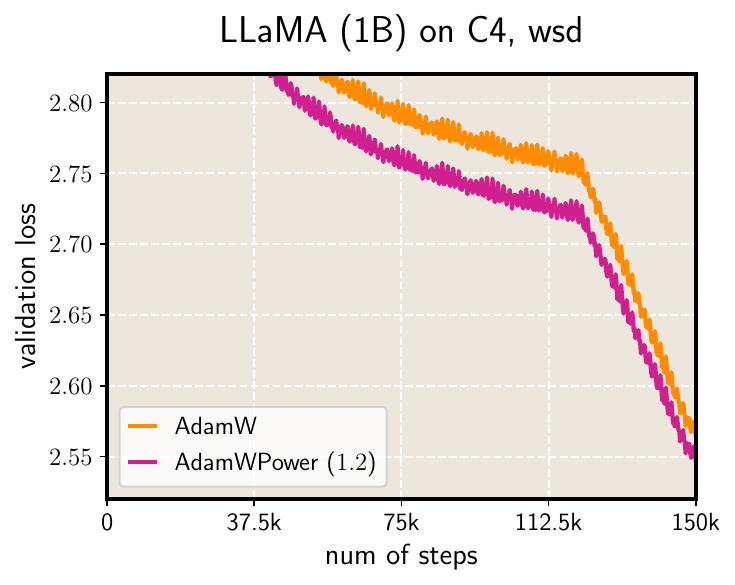}
    \includegraphics[width=0.24\linewidth]{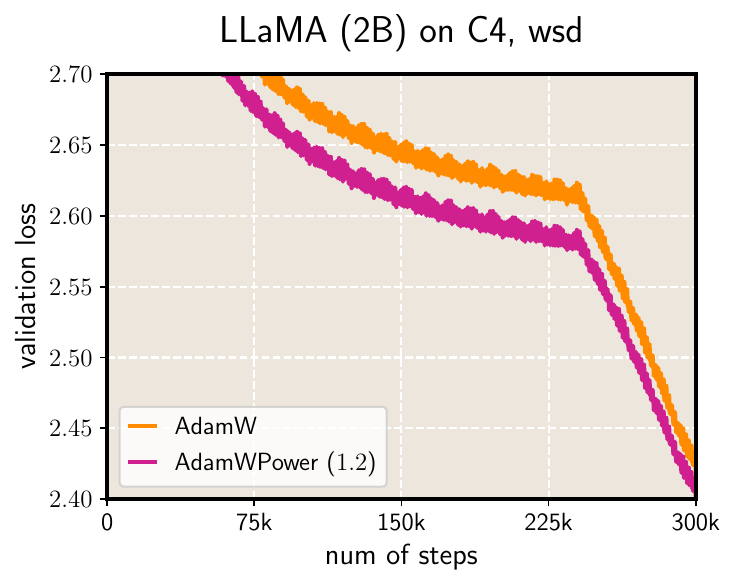}

    \vspace{-.2cm}
    
    \caption{AdamWPower ($p=1.2$) consistently outperforms AdamW in LLaMA pre-training tasks across a range of model sizes, datasets and LR schedulers.}
    \label{fig: llama on c4, full}
     \vspace{-.2cm}
\end{figure*}

\begin{table*}[!htb]
    \centering
    \caption{\small The evaluation results of LLaMA (2B) models pre-trained using the C4 dataset. The best scores in each column are bolded.}
    \small
    \begin{tabular}{l|c|c|c|c|c|c|c}
    \hline\hline
        \textsc{Method} & \textsc{ARC-E} & \textsc{ARC-C} & \textsc{PIQA} & \textsc{HellaSwag} & \textsc{OBQA} & \textsc{WinoGrande} & \textsc{Avg.} \\ \hline
        AdamW & 60.02 & \textbf{26.45} & 73.56 & 44.65 & 24.80 & 56.83 & 47.72 \\ 
        AdamWPower (1.2) & \textbf{60.35} & 26.28 & \textbf{73.61} & \textbf{44.93} & \textbf{25.00} & \textbf{59.43} & \textbf{48.26} \\ \hline\hline
    \end{tabular}

    \label{tab: downstream}

    \vspace{-.2cm}
    
\end{table*}


{\bf Main findings.} Figure~\ref{fig: llama on c4, full} compares the performance of AdamWPower (with $p=1.2$) to that of vanilla AdamW across  a range of settings, including LLaMA models of size 66M, 0.2B, 0.4B, 1B and 2B; both \texttt{cos} and \texttt{wsd} LR schedulers; and the C4 and OpenWebText datasets.
Across all experiments, AdamWPower {\bf consistently achieves a lower
terminal loss} than well-tuned AdamW baseline.
To further assess its scalability, we visualize the {\bf scaling laws} of AdamWPower versus AdamW in Figure~\ref{fig: scaling law on c4} (left and middle).
We observe that the performance gain of AdamWPower over AdamW remains consistent across a wide range of model scales, {\bf highlighting the potential scalability of AdamWPower}.

{\bf Evaluation on downstream tasks.} Additionally, we also evaluate zero-shot performances of our method on common benchmarks including ARC~\citep{yadav2019quick}, PIQA~\citep{bisk2020piqa}, HellaSwag~\citep{zellers2019hellaswag}, OBQA~\citep{mihaylov2018can}, WinoGrande~\citep{sakaguchi2021winogrande}, using the lm-evaluation-harness codebase~\citep{eval-harness}. The results are reported in Table~\ref{tab: downstream}. The model pre-trained with AdamWPower outperforms that trained with AdamW on five out of six tasks, as well as on the overall average score, demonstrating improved downstream performance under the same number of pre-training steps.

\subsection{Results on MoE Models}
\label{subsec: moe results}


Mixture-of-experts (MoE) architectures have become a key design choice in  modern LLMs, as exemplified by Qwen-2.5~\citep{yang2024qwen2} and DeepSeek-V3~\citep{liu2024deepseek}.  Compared to dense models, MoE models often exhibit greater training instability. To assess whether the benefits of AdamWPower extend to MoE models, we conduct experiments on Qwen2MoE~\citep{yang2024qwen2technicalreport}.


{\bf Main findings.}
Figure~\ref{fig: moe on c4, full} compares the performance of AdamWPower ($p=1.2$) and standard AdamW for pre-training Qwen2MoE models of sizes 0.5B, 1B, and 2B on the C4 dataset, using the \texttt{wsd} scheduler.
Across all settings, AdamWPower \textbf{consistently achieves a lower terminal loss} than the well-tuned AdamW baseline.
To further examine scaling behavior, Figure~\ref{fig: scaling law on c4} (right) visualizes the \textbf{scaling laws} of AdamWPower versus AdamW during Qwen2MoE pre-training.
The performance gap between the two optimizers remains stable across model scales, with the corresponding scaling curves remaining nearly parallel -- \textit{suggesting that the gains offered by AdamWPower may persist  at larger model scales}.

\begin{figure*}[!htb]
    \centering
    \includegraphics[width=0.24\linewidth]{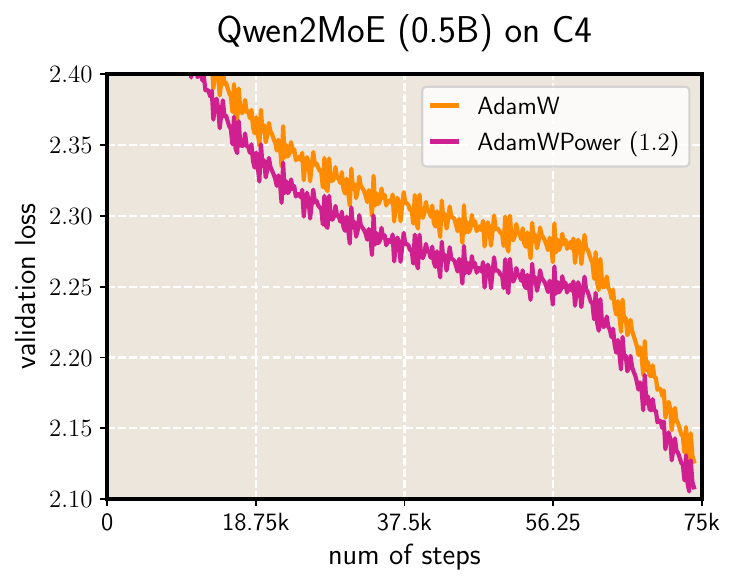}
    \includegraphics[width=0.24\linewidth]{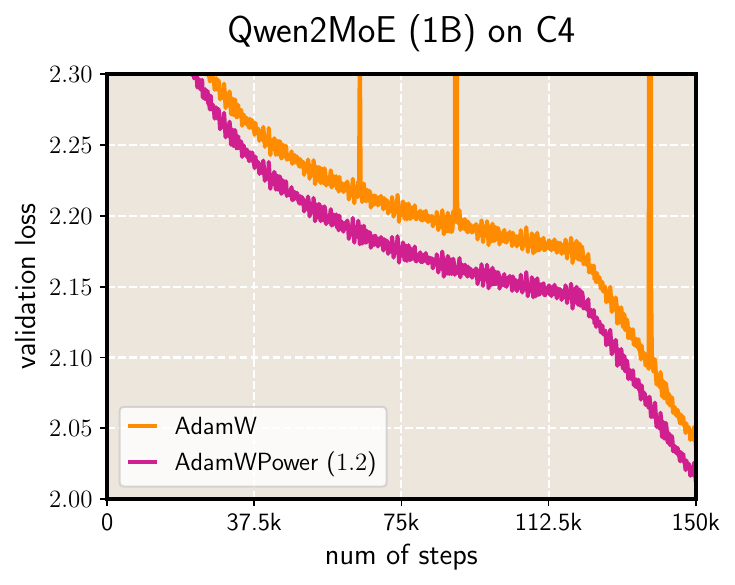}
    \includegraphics[width=0.24\linewidth]{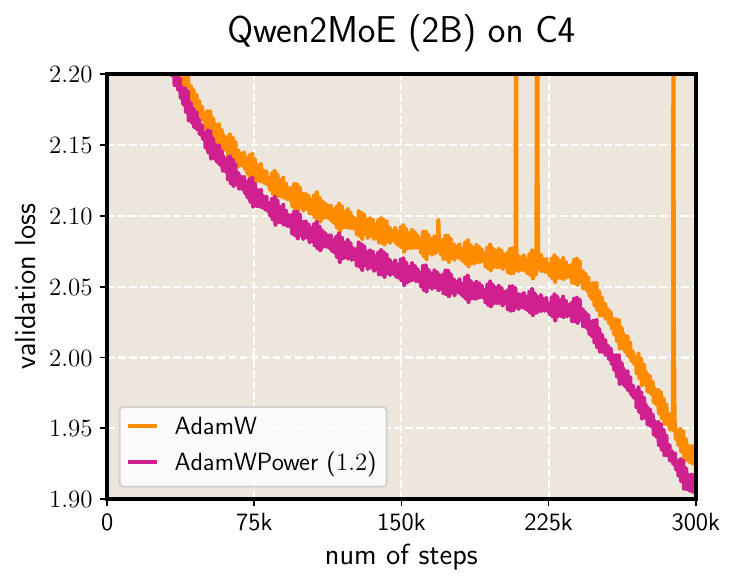}

    
    \caption{AdamWPower ($p=1.2$) consistently outperforms AdamW in Qwen2MoE pre-training tasks on C4, across varying model sizes. The learning rate schedule is \texttt{wsd}.}
    \label{fig: moe on c4, full}

    \vspace{-.3cm}
    
\end{figure*}

\begin{figure*}[!htbp]
    \centering
    \includegraphics[width=0.24\linewidth]{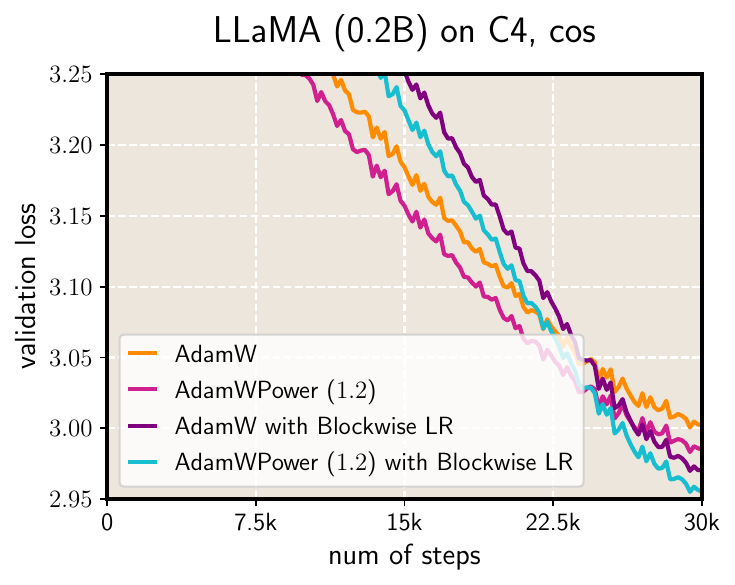}
    \includegraphics[width=0.24\linewidth]{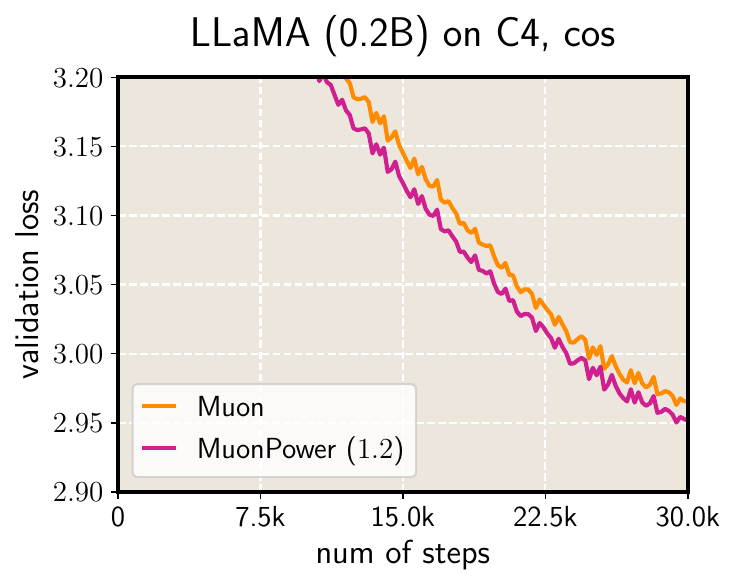}
    \includegraphics[width=0.24\linewidth]{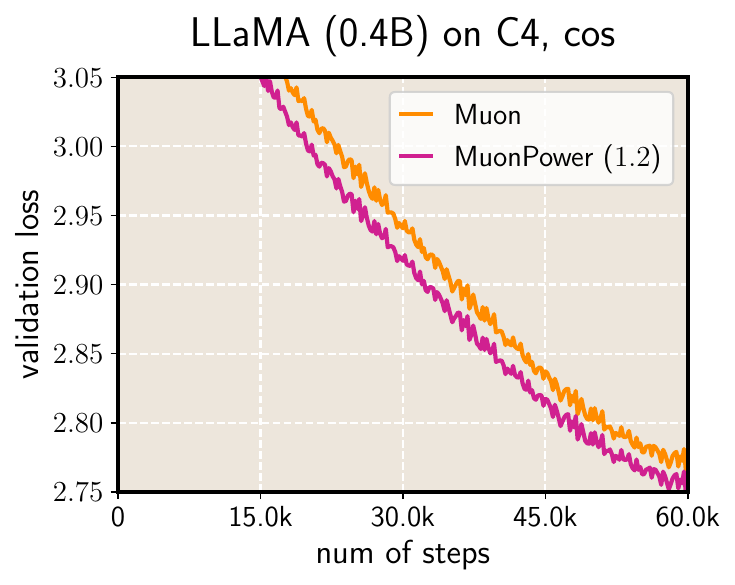}
    \caption{(left) AdamWPower with Blockwise LR outperforms both AdamWPower and AdamW with Blockwise LR in LLaMA pre-training. (middle, right) MuonPower (with $p=1.2$) outperforms Muon in LLaMA pre-training.}
    \label{fig: blockwise lr and muon}
    \vspace{-.2cm}
\end{figure*}

{\bf Special potential for MoE models.}
Additionally, we observe three surprising phenomena, suggesting that AdamWPower may offer unique advantages for MoE model training:

\vspace{-.1cm}

\begin{itemize}[leftmargin=1em]
    \item 
    Although the power $p=1.2$ was originally tuned for LLaMA, it generalizes well to Qwen2MoE models without further tuning. (it is likely that an even better $p$ exists for MoE-specific training.)
    Remarkably,  the absolute improvement achieved by AdamWPower on Qwen2MoE-2B ($0.028$) is {\bf more significant} than that on LLaMA-2B ($0.022$). Notably, Qwen2MoE-2B reaches a much lower loss ($1.93$) compared to LLaMA-2B ($2.43$), making further improvements more challenging -- yet AdamWPower still yields remarkable gains.
    \item 
    AdamWPower also exhibits improved {\bf training stability}, reducing the occurrence of loss spikes seen with AdamW. This effect is particularly visible in the 1B and 2B curves in Figure~\ref{fig: moe on c4, full} (middle, right).
    Based on recent understanding in Section~\ref{sec: related works}, the fast vibrations along the sharp (valley) directions mainly decide the training (in)stability.
    We {\em hypothesize} that the gradient power transformation in AdamWPower may help suppress the vibrations along these directions. We leave a detailed investigation of this phenomenon to future work.
    \item The \texttt{wsd} scheduler has become increasingly popular in recent LLM pre-training~\citep{liu2024deepseek,hagele2024scaling}, always taking a long stable phase.
We observe that the advantage of AdamWPower {\bf gradually increases} {\em throughout the LR stable phase}. This suggests that AdamWPower may be particularly suited for modern training pipelines that adopt \texttt{wsd} schedules.
\end{itemize}

\vspace*{-.2cm}

\subsection{Compatibility with Other Optimizers}\label{subsec: blockwise lr and muon}

As discussed in Section~\ref{sec: related works}, several optimizers have recently been proposed to enhance LLM pre-training. 
While AdamWPower has demonstrated superiority over AdamW in both dense and MoE models, we now ask: {\em can GradPower  also improve the performance of other state-of-the-art optimizers?}

To investigate this, we focus on two representative optimizers: AdamW with {\bf Blockwise LR}~\citep{wang2025sharpness} and {\bf Muon optimizer}~\citep{jordan2024muon,liu2025muon}.
Blockwise LR assigns separate learning rates to different Transformer blocks and has shown substantial improvements over standard AdamW.
Muon, on the other hand, breaks away from the AdamW framework entirely and has recently been shown to achieve better scaling laws than AdamW~\citep{liu2025muon}. 
We refer to the application of GradPower to Muon as {\bf MuonPower}.

The results shown in Figure~\ref{fig: blockwise lr and muon}, highlight two key findings.
\textbf{(i) AdamWPower with Blockwise LR} achieves a lower terminal loss than both AdamWPower and AdamW with Blockwise LR alone.
Notably, the observed improvement ($0.045$) is \emph{nearly the sum} of the gains from AdamWPower alone ($0.015$) and Blockwise LR alone ($0.03$), suggesting that their benefits are largely orthogonal.
\textbf{(ii) MuonPower} ($p=1.2$), the GradPower-augmented Muon, also outperforms the well-tuned Muon baseline.
These results demonstrate the versatility of GradPower as a general enhancement that can be seamlessly integrated into other optimizers.



\subsection{Interaction between GradPower and Clipping}

Gradient clipping is a standard technique for stabilizing training in LLM pre-training~\citep{shoeybi2019megatron,touvron2023llama}. 
In our default implementation, gradient clipping is applied before the GradPower transformation.

Here, we examine whether the {\em ordering between gradient clipping and GradPower} impacts performance.
We conduct a controlled experiment based on the setting of Figure~\ref{fig: optimal p, 0.2B on C4} on LLaMA (0.2B). Specifically, we switch the order of gradient clipping and the GradPower transformation. We refer to this variant as AdamWPower-II, in contrast to the standard AdamWPower implementation.
As shown in Figure~\ref{fig: optimal p, batch size, llama} (left), the training curves are nearly indistinguishable throughout training, indicating that the ordering of clipping and GradPower has negligible effect on convergence behavior. 
Crucially, both procedures ensure that the final gradient updates remain bounded, preserving training stability.

\begin{figure}[!htbp]
    \centering
    \includegraphics[width=0.52\linewidth]{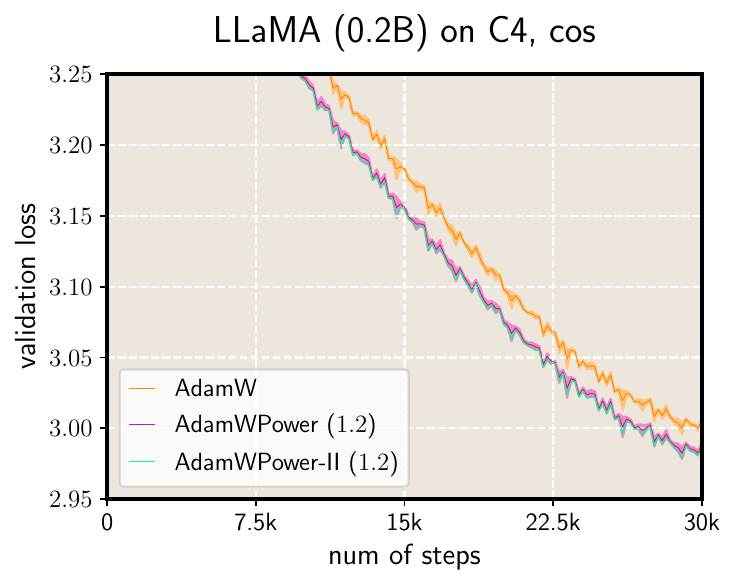}
    \hspace{-.2cm}
    \includegraphics[width=0.465\linewidth]{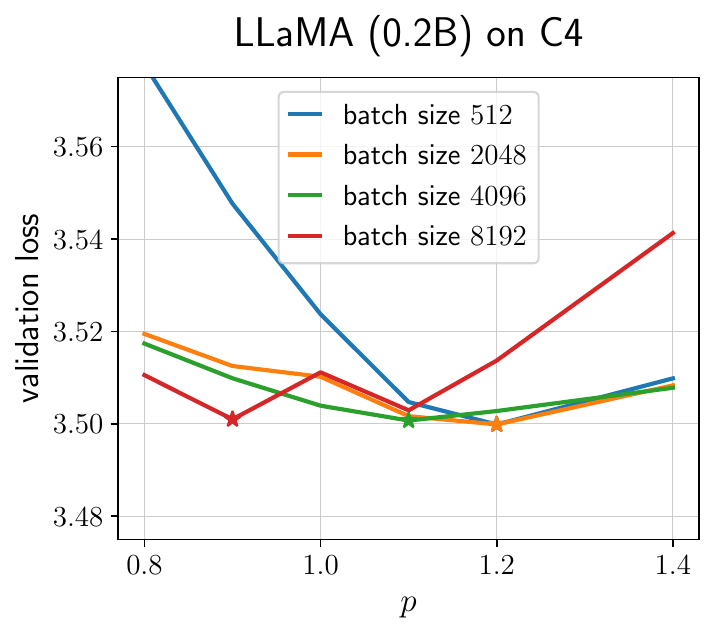}
    \caption{\small (left) 
    Comparison of gradient clipping before vs. after applying GradPower in LLaMA pre-training. The performance is virtually indistinguishable. The shaded regions in the plots denote the standard deviation. (right) The influence of batch size for the optimal power $p$ in LLaMA pre-training tasks.}
    \label{fig: optimal p, batch size, llama}
    \vspace{-.4cm}
\end{figure}


\subsection{Influence of Batch Size}
\label{subsec: experiment: batch size}

Finally, we investigate how batch size influences the performance of GradPower.
Batch size plays a critical role in deep learning, with larger batch sizes producing lower gradient noise and more accurate gradient~\citep{keskar2016large,mccandlish2018empirical,wang2026fast}.

Unlike the previous experimental settings, here we conduct the experiments on C4 dataset, varying the batch size from the standard 512 up to 8192.
For each batch size, we evaluate AdamWPower with multiple values of $p$, and record their validation loss of when the optimal validation loss reaches approximately $3.5$. The experimental details are provided in Appendix~\ref{appendix: experiments}.

{\bf Main findings.}
The results, shown in Figure~\ref{fig: optimal p, batch size, llama} (right), demonstrate a clear trend: the optimal power $p$ decreases as batch size increases, i.e., as the gradient noise level decreases. This finding reveals a strong correlation between batch size and the optimal power $p$ in AdamWPower.
For standard (small) batch sizes, the optimal power $p$ tends to be greater than $1$; in contrast, for large batch sizes, the optimal power $p$ might fall below $1$. 

{\bf Vision tasks.}
We also conduct the experiments using ResNet-34 model~\citep{he2016deep} on CIFAR-10 dataset~\citep{krizhevsky2009learning}, varying the batch size from 32 to 128. The results in Table~\ref{tab: cv p} further validates above point. Moreover, it demonstrates the generalizability of our method beyond language model pre-training.

In the next section, we provide a theoretical explanation for this phenomenon.

\begin{table}[!htbp]
\vspace{-.1cm}
    \centering
    \caption{\small The influence of batch size for the optimal power $p$ in vision tasks.}
    \vspace{-.1cm}
    \footnotesize
    \begin{tabular}{c|c|c|c}
    \hline\hline
        batch size & 128 & 64 & 32 \\ \hline
        $p=0.8$ & \textbf{94.35} & 94.22 & 94.04 \\ \hline
        $p=0.85$ & 94.27 & \textbf{94.40} & 94.07 \\ \hline
        $p=0.9$ & 94.22 & 94.22 & 94.15 \\ \hline
        $p=1.0$ & 93.98 & 94.10 & \textbf{94.30} \\ \hline
        $p = 1.1$ & 93.38 & 93.97 & 94.25 \\ \hline
        $p = 1.2$ & 93.15 & 93.77 & 93.85 \\ \hline\hline
    \end{tabular}
    \vspace{-.1cm}
    \label{tab: cv p}
\end{table}


\section{Theoretical Insights}
\label{sec: theory-whole-sec}

\subsection{An Illustrative Case Study}
\label{subsec: case study}


This subsection investigates a phenomenological example, both theoretically and empirically, to illustrate how varying the power $p$ in AdamWPower affects the update magnitude.
Motivated by the empirical findings in Section~\ref{subsec: experiment: batch size}, which show that batch size (gradient noise) affects the optimal value of $p$, we study our example under varying signal-to-noise regimes.

{\bf Slow dynamics along flat directions.}
As discussed in Section~\ref{sec: related works}, recent studies have revealed that the landscape in LLM training can be decomposed into flat and sharp directions, also referred to as river and valley components~\citep{wen2024understanding}.
The loss along river component typically determines the loss at the bottom of the landscape.
Along these {\em flat directions}, the optimizer tends to make {\em slow but steady progress}, and appears to {\em remain aligned} for a period of time. 

{\bf The anisotropy of stochastic gradients.}
Several works in Section~\ref{sec: related works} show that stochastic gradients closely reflect the curvature structures: {\em flat (low-curvature) directions} exhibit {\em small stochastic gradient magnitudes}, whereas sharp (high-curvature) directions exhibit larger stochastic gradients.

Motivated by these insights, we consider a one-dimensional example to study whether varying $p$ in AdamWPower can {\em accelerate these slow dynamics along the flat directions}, thereby leading to more efficient loss descent.

\begin{example}\label{example: toy model}
    For simplicity, consider a 1-dimensional flat direction. Let the stochastic gradients at time $t \in\bbN$ follow $g_t\overset{\rm i.i.d.}{\sim}{\rm Unif}(\mu-\sigma,\mu+\sigma)$, where $0<\mu,\sigma\ll 1$\footnote{Empirical studies suggest that gradient scales in LLM training are often very small~\citep{huang2025spam}.}. Here, $\mu$ reflects the full-batch gradient, and $\sigma$ captures the stochastic noise level.
\end{example}

\begin{wrapfigure}{r}{0.22\textwidth}
    \vspace{-.7cm}
    \centering
    \includegraphics[width=0.21\textwidth]{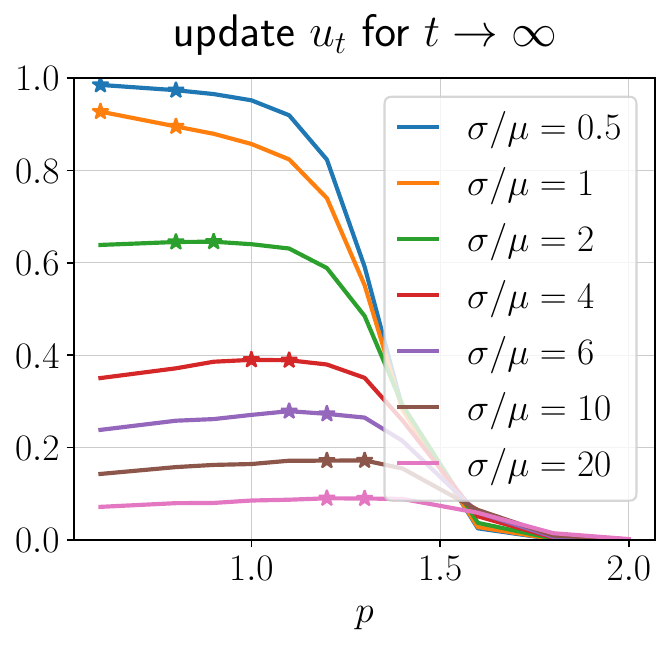}
    \vspace{-.2cm}
    \caption{\small Numerical results for Example~\ref{example: toy model}.
    We plot the value of $u_t$ at $t=10^6$ for AdamWPower across different $p$'s under varying noise-to-signal ratios. 
    For each curve, the optimal and suboptimal $p$ values are marked with stars. 
    The $\mu$ is set to $\mu=10^{-6}$. Other hyperparameters follow standard values: $\beta_1=0.9$, $\beta_2=0.95$, $\epsilon=10^{-8}$, and $\lambda=0$. The learning rate $\eta$ does not affect the result.}
    \label{fig: toy example, optimal p}
    \vspace{-1.2cm}
\end{wrapfigure}
Our goal is to investigate the values of $p$ that maximize the update magnitude $u_t=m_t/(\sqrt{v_t} + \epsilon)$ in AdamWPower (Alg.~\ref{alg: adamwpower}). For simplicity, we set weight decay to $0$. We now present both empirical and theoretical analysis.

\paragraph{Empirical findings.}
We begin by numerically simulating the update $u_t$. The results are presented in Figure~\ref{fig: toy example, optimal p}. Notably, the optimal value of $p$ varies across noise-to-signal regimes, exhibiting two distinct behaviors:
\begin{itemize}[leftmargin=1em]
    \item {\em Low-noise regime} $\sigma/\mu\leq 1$ (blue and orange curves), it is clear that the update magnitude decreases monotonically with increasing $p$, and the optimal power is small, satisfying $p^\star<1$.
    \item {\em High-noise regime} $\sigma/\mu>1$, the update magnitude increases and then decreases with increasing $p$. Moreover, for noise-dominant regime, the optimal power satisfies $p^\star>1$ (red, purple, brown, and pink curves).  
\end{itemize}

These findings closely align with our empirical results in real-world LLM pre-training tasks in Section~\ref{sec: experiments}. 
As the batch size increases (corresponding to lower gradient noise), the optimal power $p^\star$ decreases accordingly, transitioning from $p^\star>1$ to $p^\star<1$, as observed in Section~\ref{subsec: experiment: batch size}. 
Remarkably, the optimal power $p^\star = 1.2$ in the high-noise regime matches the value used across most LLM pre-training experiments in Section~\ref{sec: experiments}.

\paragraph{Theoretical analysis.}
To better understand these interesting behaviors, we theoretically analyze this problem.
To facilitate analytical derivation, we consider the limiting case where $\beta_2\to 1$, which closely approximates typical settings in practice (e.g., $0.95$ or $0.999$). 
We define the limiting update of AdamWPower as:
\begin{equation}
    u:=\lim_{t\to\infty}\lim_{\beta_2\to 1} u_t,
\end{equation} 
where $u_t=m_t/(\sqrt{v_t}+\epsilon)$, with $m_t=\texttt{EMA}_{\beta_1}(\{g_s^p\}_1^t)$ and $v_t=\texttt{EMA}_{\beta_2}(\{(g_s^{p})^2\}_1^t)$.
In this limit, we obtain the {\em closed-form expression}: $u=\bbE[g^p]/(\sqrt{\bbE[(g^{p})^2]}+\epsilon)$, {\em a.s.}, $g\sim{\rm Unif}(\mu-\sigma,\mu+\sigma)$.
This formulation allows explicit computation and facilitates verification of the empirical trends. We present two propositions corresponding to the low-noise and high-noise regimes.

\begin{proposition}[low-noise regime, $\sigma\ll\mu$]\label{prop: low-noise regime}
It holds that $u=\frac{1+o(1)}{1+\frac{\epsilon}{\mu^p}}$, {\em a.s.}. 
Letting $\tilde{u}=\frac{1}{1+\frac{\epsilon}{\mu^p}}$,
we observe that $\tilde{u}$ is monotonically decreasing w.r.t. $p$.
\end{proposition}


This proposition quantitatively explains the monotonicity observed in the low-noise regime. Furthermore, it shows that 
the maximum update is approximately $\frac{1}{1+\epsilon}\approx1$, achieved in the limit as $p\to0$. This aligns with Figure~\ref{fig: toy example, optimal p}.

\begin{proposition}[high-noise regime, $\mu\ll\sigma$]
\label{prop: high-noise regime}
It holds that $u=\frac{\mu}{\sigma}\frac{1+o(1)}{\frac{1}{\sqrt{2p+1}}+\frac{\epsilon}{\sigma^p}}$, {\em a.s.}. Letting $\tilde{u}=\frac{\mu}{\sigma}\frac{1}{\frac{1}{\sqrt{2p+1}}+\frac{\epsilon}{\sigma^p}}$, we observe the following:
If $\epsilon\log(1/\sigma)<1$, then there exists an optimal power $p^\star$ such that $\tilde{u}$ increases for $0<p<p^\star$ and decreases for $p>p^\star$. Moreover, we have a tight estimate: $p^\star=\Theta\left(\frac{\log(\epsilon\log(1/\sigma))}{\log\sigma}\right)$.
\end{proposition}


Notably, in practice, $\epsilon$ is typically chosen sufficiently small (e.g., $\epsilon\ll\sigma$), ensuring $\frac{\log(\epsilon\log(1/\sigma))}{\log\sigma}>1$. This again aligns with our empirical observation that $p^\star>1$ in the high-noise regime.

The intuition behind Proposition~\ref{prop: high-noise regime} is as follows.  
When $p$ is relatively small, the denominator is dominated by $\sqrt{\bbE[(g^{p})^2]}$. Since $g\ll1$, increasing $p$ reduces both the numerator $\bbE[g^{p}]$ and denominator $\sqrt{\bbE[(g^{p})^2]}$. In the high-noise regime, the reduction in the denominator outweighs that in the numerator, resulting in a larger update.
In contrast, when $p$ is relatively large, the denominator is dominated by $\epsilon$, and AdamWPower degenerates to SGDPower, where the update is approximately $\bbE[g^p]/\epsilon$.
In this regime, increasing $p$ reduces the update magnitude.

Although the above example is synthetic, it reveals several non-trivial phenomena highly aligned with LLM pre-training tasks, particularly the existence and behavior of the best $p^\star$ across noise-to-signal regimes. These insights deepen our understanding of how GradPower influences the performance of AdamWPower and suggest practical guidance for selecting $p$.


\subsection{Convergence Guarantees}
\label{subsec: theory}


Let $\cL:\bbR^d\to\bbR$ be a non-convex loss function. For any $\btheta\in\bbR^d$, let $\bg(\btheta)$ denote the stochastic gradient satisfying $\bbE[\bg(\btheta)]=\nabla\cL(\btheta)$.

In this subsection, we consider the classical setting of smooth non-convex optimization and investigate the theoretical benefits of applying GradPower within adaptive optimizers.
Since the analysis of Adam is technically complex, to gain clear theoretical insights, 
we instead analyze its predecessor, Adagrad, a foundational adaptive optimization algorithm~\citep{duchi2011adaptive}. 
The update rule of {\bf AdagradPower} (Adagrad using GradPower) is given by:

\vspace{-.3cm}

\begin{equation}\label{equ: adagradpower}
    \btheta_{t+1}=\btheta_t-\eta\frac{\bg_t^p}{\sqrt{\bv_t+\epsilon}},\quad
    \bv_t=\sum_{s=1}^t(\bg_t^p)^2,
\end{equation}


where the power $p>0$, and we denote $\bg_t$ as the stochastic gradient $\bg(\btheta_t)$ for simplicity.

To establish the convergence results, we adopt the following standard assumptions, consistent with Section 2.3 in~\citet{defossez2020simple}.

\begin{assumption}[\citet{defossez2020simple}]\label{ass: convergence}
The following conditions hold:


\begin{itemize}[leftmargin=1em]
    \item {\em $\cL$ is bounded below by $\cL^\star$}, i.e., for all $\btheta\in\bbR^d$, $\cL(\btheta)\geq\cL^\star$.
    \item {\em The loss function is $H$-smooth}, i.e., there exists a constant $H>0$ such that for all $\btheta,\btheta'\in\bbR^d$, $\|\nabla\cL(\btheta)-\nabla\cL(\btheta')\|_2\leq H\|\btheta-\btheta'\|_2$.
    \item {\em The $\ell_\infty$ norm of the stochastic gradients is uniformly almost surely bounded}, i.e., there exists a constant $R>0$ such that for all $\btheta\in\bbR^d$, $\|\bg(\btheta)\|_{\infty}+\epsilon\leq R$, {\em a.s.}.
\end{itemize}    
\end{assumption}


Under this assumption, the convergence guarantee of Adagrad is well established:

\begin{theorem}[Adagrad; Theorem 1 in~\citet{defossez2020simple}]\label{thm: adagrad}
Suppose Assumption~\ref{ass: convergence} holds. Let $\{\btheta_t\}_{t=0}^T$ are trained by {\bf Adagrad}~\eqref{equ: adagradpower} with $p=1$. Then for any $T\in\bbN$, we have:

\vspace{-.4cm}

\begin{equation}\label{equ: adagrad rate}
    \begin{split}
    &\!\min_{1\leq t\leq T}\bbE\left[\|\nabla\cL(\btheta_t)\|_2^2\right]
    \\\!\leq&
    \frac{2R(\cL(\btheta_0)\!-\!\cL^\star)}{\eta\sqrt{T}}\!+\!\frac{Rd(4R+\eta H)\log(1\!+\! R^2 T/\epsilon)}{\sqrt{T}}.
\end{split}
\end{equation}

\end{theorem}

We now study the convergence of AdagradPower in both low-noise and high-noise regimes.


{\bf Low-noise regime.} 
We introduce an additional assumption about the noise scale.

\begin{assumption}[Low-noise regime]\label{ass: low-noise regime}
There exist constants $p\in(0,1)$ and $c>0$ such that $\bbE[g_i^p(\btheta)]\nabla_i\cL(\btheta)\geq c|\nabla_i\cL(\btheta)|^{p+1}$ holds for all $\btheta\in\bbR^d$ and $i\in[d]$.
\end{assumption}


This assumption is satisfied in many low-noise scenarios: 
\begin{example}\label{example: low noise}
(I) Deterministic regime (the limit case of low noise): if $g_i(\btheta)=\nabla_i\cL(\btheta)$, then Assumption~\ref{ass: low-noise regime} holds for all $p\in(0,1)$ with $c=1$.
(II) Uniform distribution: if $g_i\sim{\rm Unif}(\nabla_i\cL-\sigma,\nabla_i\cL+\sigma)$ with $\sigma\ll|\nabla_i\cL|$, then Assumption~\ref{ass: low-noise regime} holds for all $p\in(0,1)$ as
$\bbE[g_i^p]\nabla_i\cL=|\nabla_i\cL|^{p+1}(1+o(\sigma/|\nabla_i\cL|))\geq0.99|\nabla_i\cL|^{p+1}$.
\end{example}

\begin{theorem}[AdagradPower, low-noise regime]\label{thm: adagradpower, low-noise}
Suppose Assumption~\ref{ass: convergence} and~\ref{ass: low-noise regime} hold, as well as $R<1$\footnote{Empirical studies suggest that gradient scales in LLM training are often very small~\citep{huang2025spam}.}.
Let $\{\btheta_t\}_{t=0}^T$ are trained by {\bf AdagradPower}~\eqref{equ: adagradpower}, with the power $p\in(0,1)$ as given in Assumption~\ref{ass: low-noise regime}.
Then for any $T\in\bbN$, we have:

\vspace{-.6cm}

\begin{equation}
    \min_{1\leq t\leq T}\bbE\left[\|\nabla\cL(\btheta_t)\|_{p+1}^{p+1}\right]\leq
    \cO\left(\frac{\log T}{\sqrt{T}}\right).
\end{equation}
\end{theorem}



by Jensen, Theorem~\ref{thm: adagrad} yields the rate for Adagrad: $\min_t \mathbb{E}||\nabla L(\theta_t)||_{p+1}^{p+1}\leq O(\log^{(p+1)/2}T/T^{(p+1)/4}),0<p<1$. In comparison, Theorem~\ref{thm: adagradpower, low-noise} shows that AdagradPower has the rate $O(\log T/T^{1/2})$. AdagradPower achieves a convergence rate $2/(p + 1)$ times faster than Adagrad. 
For Example~\ref{example: low noise}, this yields nearly a $2\times$ acceleration for $p \to 0$.
This result is consistent with observations in Section~\ref{subsec: case study} and~\ref{subsec: experiment: batch size} that the optimal power $p$ for adaptive optimizers is less than $1$ in the low-noise regime.
The proof is presented in Appendix~\ref{appendix: proof of theory}.

{\bf High-noise regime.} We introduce an additional assumption regarding the noise scale:

\begin{assumption}[High-noise regime]\label{ass: high-noise regime}
(C1) There exist constants $p>1,\sigma>0$ such that $\bbE[g_i^p(\btheta)]\nabla_i\cL(\btheta)\geq\sigma|\nabla_i\cL(\btheta)|^2$ holds for all $\btheta\in\bbR^d$ and $i\in[d]$.
(C2). It holds that $\sigma>R^{p-1}$.
\end{assumption}


The first condition asserts that the gradient noise is non-degenerate. 
The second condition further asserts that the gradient noise is in a high level.
Notably, these conditions are naturally satisfied in many high-noise settings:

\begin{example}\label{example: binary distribution}
Consider $g_i$ satisfy binary distribution $\bbP(g_i=\nabla_i\cL-\sigma_i)=\bbP(g_i=\nabla_i\cL+\sigma_i)=\frac{1}{2}$. 
Then for any odd number $p>1$, $\bbE[g_i^p]\nabla_i\cL\geq p\sigma_i^{p-1}|\nabla_i\cL|^2$. Thus, (C1) in Assumption~\ref{ass: high-noise regime} holds with $\sigma=p\sigma_i^{p-1}$.
As for (C2), in high-noise regime with $|\nabla_i\cL|\ll\sigma_i$, we have $\frac{R^{p-1}}{\sigma}\leq\frac{(|\nabla_i\cL|+\sigma_i)^{p-1}}{p\sigma_i^{p-1}}\leq\frac{1.01}{p}<1$.
\end{example}

\begin{theorem}[AdagradPower, high-noise regime]\label{thm: adagradpower, high-noise}
Suppose Assumption~\ref{ass: convergence} and~\ref{ass: high-noise regime} hold, as well as $R<1$.
Let $\{\btheta_t\}_{t=0}^T$ be trained by {\bf AdagradPower}~\eqref{equ: adagradpower}, with the power $p>1$ as given in Assumption~\ref{ass: high-noise regime}. 
Then for any $T\in\bbN$, we have:

\vspace{-.6cm}

\begin{equation}
    \min_{1\leq t\leq T}\bbE\left[\|\nabla\cL(\btheta_t)\|_2^2\right]\leq
    \frac{R^{p-1}}{\sigma}\cdot\Big({\rm R.H.S.} \text{ of~\eqref{equ: adagrad rate}} \Big),
\end{equation}

\vspace{-.3cm}
where ${R^{p-1}}/{\sigma}<1$.

\end{theorem}
Comparing Theorems~\ref{thm: adagrad} and~\ref{thm: adagradpower, high-noise}, we observe that
AdagradPower accelerates convergence of Adagrad by a constant factor $\frac{R^{p-1}}{\sigma}$ in high-noise regime.
For Example~\ref{example: binary distribution}, the acceleration is significant, due to $\frac{R^{p-1}}{\sigma}\leq\frac{1.01}{p}$ for any positive odd $p$. 
This result provides theoretical support for the empirical superiority of adaptive optimizers using GradPower in LLM pretraining in Section~\ref{sec: experiments}
Notably, the theoretical insights are highly aligned with those in Proposition~\ref{prop: high-noise regime}. In the high-noise regime, using $p>1$ reduces both numerator $\bg_t$ and denominator $\sqrt{\bv_t+\epsilon}$. However, reduction in the denominator outweighs that in the numerator, resulting in a faster convergence speed. 
The formal proof refines this argument and is presented in Appendix~\ref{appendix: proof of theory}.



\section{Conclusion}
\label{sec: Conclusion}


We propose GradPower, a simple yet effective method  for improving the efficiency of gradient-based optimizers.
Experimentally, AdamWPower (AdamW using GradPower) consistently achieves lower terminal loss and improved scaling laws than AdamW across various LLM pre-training tasks.

For future work, it would be interesting to investigate why AdamWPower exhibits particular potential for MoE models and \texttt{wsd} LR scheduler. 
Experimentally, exploring the applicability of GradPower beyond LLMs, as well as its integration with other optimizers, could further extend its impact. In addition, developing a dynamic schedule for the GradPower exponent $p$, adapted to the evolving SNR throughout training, presents both a challenging and a potentially valuable direction.

\section*{Acknowledgments}
Lei Wu is supported by the National Natural Science Foundation of China (NSFC12522120, NSFC92470122, and NSFC12288101).
Mingze Wang is supported by Young Scientists (PhD) Fund of the National Natural Science Foundation of China (No.~124B2028).

\section*{Impact Statement}


This paper presents work whose goal is to advance the field of deep learning, with a specific focus on improving large language model (LLM) pre-training. While our work has the potential to impact society in various ways, we do not identify any immediate negative societal impacts requiring special mitigation beyond standard responsible use.


\bibliography{ref}
\bibliographystyle{icml2026}

\newpage
\appendix
\onecolumn



\begin{center}
    \noindent\rule{\textwidth}{1.0pt} 
    \vspace{-0.25cm}
    \LARGE \textbf{Appendix} 
    \noindent\rule{\textwidth}{1.0pt}
\end{center}

\startcontents[sections]
\printcontents[sections]{l}{1}{\setcounter{tocdepth}{2}}

\vspace{1.cm}

\section{Relationship with Other Similar Works.}

{\bf Powerball method.} After completing this work, we found that the Powerball method~\citep{yuan2019powerball} shares the similar methodology as our approach. However, prior studies on Powerball method have been restricted to traditional optimizers---such as GD~\citep{yuan2019powerball}, SGD~\citep{zhou2020pbsgd,yang2024powerball}, and SARAH~\citep{qin2025stochastic}—--and evaluated primarily on relatively small-scale benchmarks including CIFAR-10, CIFAR-100 and MNIST. Although \citet{baiesi2019power} combined Powerball with Adam, the experiments were limited to small and illustrative problems. In contrast, our work focuses on modern adaptive optimizers such as Adam and Muon in the context of language model pre-training, a modern and practically important setting. Moreover, previous Powerball studies examined only the narrow regime with $p<1$, our work studies both $p<1$ and $p>1$ regimes, and further develop a comprehensive theoretical study of the relationship between optimal $p$ and batch size.

\paragraph{Powerpropagation method.}
Powerpropagation~\citep{schwarz2021powerpropagation} also employs a sign-power transformation, but applies it to the model parameters rather than to the gradients. Specifically, Powerpropagation reparameterizes weights as $\theta=\phi|\phi|^{\alpha-1}$, which induces an effective gradient scaling through the chain rule, $g_{\phi}=\alpha g_{\theta}|\phi|^{\alpha-1}$. In contrast, GradPower transformation is applied directly to gradient, $\varphi_p(g)=\mathrm{sign}(g)|g|^p$, before the base optimizer step and without modifying the model parameterization. Thus, although the two methods share a similar sign-power functional form, they operate in fundamentally different spaces and induce different mechanisms: Powerpropagation scales updates by parameter magnitude, whereas GradPower scales updates by gradient magnitude.

\section{Experimental Details}\label{appendix: experiments}

{\bf Models.} We utilize two popular classes of LLM models for our pre-training experiments:
 \begin{itemize}[leftmargin=2em]    
    \item {\bf LLaMA.} LLaMA~\citep{touvron2023llama} is a popular Dense decoder-only Transformer architecture, incorporating Rotary Positional Encoding (RoPE)~\citep{su2024roformer}, Swish-Gated Linear Unit (SwiGLU), and Root mean square layer normalization (RMSNorm). 
    We pre-train LLaMA models of sizes ranging from 66M to 2B parameters. 
    Additional model configurations are detailed in Table~\ref{table: dense model config and max lrs}.
    \item {\bf Qwen2MoE.} 
    Comparing with Llama, Qwen2MoE is decoder-only mixture-of-experts architecture. Following Qwen2MoE and~\citep{yang2024qwen2technicalreport}, we  activate 4 experts per token for all models.
   For detailed model configurations, refer to Table~\ref{table: moe model config and max lrs}.
\end{itemize}

{\bf Datasets.} Models are pre-trained on the following datasets:
\begin{itemize}[leftmargin=2em]
    \item{\bf Colossal Clean Crawled Corpus (C4)}~\citep{raffel2020exploring}. It is a large-scale public language dataset, widely used for LLM pre-training such as T5~\citep{raffel2020exploring}, and prior pre-training studies~\citep{zhao2024galore,zhao2024deconstructing}. We use the T5 tokenizer, with the vocabulary size 32100.
    \item {\bf OpenWebText}~\citep{Gokaslan2019OpenWeb}. It is an opensource recreation of the WebText corpus, is extensively utilized for LLM pre-training such as RoBERTa~\citep{liu2019roberta} and nanoGPT~\citep{Karpathy2022}. 
    Following~\citet{Karpathy2022,liu2023sophia}, we use the GPT-2 tokenizer, with the vocabulary size 50304.
\end{itemize}

{\bf LR schedulers.}
    We evaluate two popular LR scheduling strategies: 
    \begin{itemize}[leftmargin=2em]
    \item \texttt{cos} (cosine scheduler)~\citep{Karpathy2022,touvron2023llama}: a linear warm-up to peak \texttt{lr\_max}, followed by cosine decay to a terminal LR \texttt{lr\_min}.
    \item \texttt{wsd} (warmup-stable-decay scheduler)~\citep{zhai2022scaling,hu2024minicpm,hagele2024scaling}:  a linear warm-up LR to peak \texttt{lr\_max}, followed by a stable phase where LR remains at \texttt{lr\_max} (up to 80\% of the total training steps), and then a linear decay to \texttt{lr\_min}.
    \end{itemize}

All experiments are conducted on 8 A100 80G GPUs.


\begin{table}[!ht]
		\centering
        \renewcommand{\arraystretch}{1.25}
		\caption{\small Dense model configurations and optimally-tuned peak learning rates for AdamW.}
		\label{table: dense model config and max lrs}
		\begin{small}
		\begin{tabular}{l|c|c|c|c|c|c}
		\hline 
		Acronym & Size & $d_{\mathrm{model}}$ & $d_{\mathrm{FF}}$ & n$\_$head & depth &  \texttt{lr\_max} \\\hline\hline 
	LLaMA (66M) & 66M & 512  & 2048 & 8 & 8 & 1e-3 (on C4) \\
    LLaMA (0.2B) & 200M & 1024 & 4096 & 16 & 8 & 1e-3 (on C4) \\
	LLaMA (0.25B) & 237M & 1024 & 4096 & 16 & 8 & 8e-4 (on OpenWebText) \\
	LLaMA (0.4B) & 400M & 1280 & 5120 & 16 & 12 & 6e-4 (on C4) \\
	LLaMA (1B) & 1004M & 1600 & 6400 & 25 & 22 & 3e-4 (on C4) \\ 
    LLaMA (2B) & 1994M & 2048 & 8096 & 32 & 28 & 2e-4 (on C4) \\ 
	 \hline 
	\end{tabular}
	\end{small}
\end{table}

\begin{table}[!ht]
		\centering
        \renewcommand{\arraystretch}{1.25}
		\caption{\small MoE model configurations and optimally-tuned peak learning rates for AdamW on C4.}
		\label{table: moe model config and max lrs}
		\begin{small}
		\begin{tabular}{l|c|c|c|c|c|c|c|c}
		\hline 
		Acronym & Size & Activated Size & $d_{\mathrm{model}}$ & $d_{\mathrm{FF}}$ & n$\_$head & depth & n$\_$experts & \texttt{lr\_max} \\
		\hline\hline 
	Qwen2MoE (0.5B) & 502M & 247M &  768 & 3072 & 12 & 12 & 16 & 6e-4 \\
	Qwen2MoE (1B) & 1040M & 297M & 768 & 3072 & 12 & 15 & 32 & 3e-4 \\
	Qwen2MoE (2B) & 1945M & 536M & 1024 & 4096 & 16 & 16 & 32 & 2e-4 \\
	 \hline 
		\end{tabular}
		\end{small}
\end{table}

For the vision experiment, we used the standard 34 layer ResNet model~\citep{he2016deep} on the CIFAR-10 dataset~\citep{krizhevsky2009learning}. We use AdamW optimizer and the commonly used \texttt{cos} learning rate scheduler.

\subsection{Experimental details for Section~\ref{subsec: dense results} and~\ref{subsec: moe results}}

{\bf AdamW baselines.} 
We use the standard Adam optimizer (with decoupled weight decay) as the baseline in most experiments (except Section~\ref{subsec: blockwise lr and muon}). 
The baseline is configured with hyperparameters $\beta_1=0.9,\beta_2=0.95$, weight decay $\lambda=0.1$, and gradient clipping threshold of $1.0$, following protocols used in LLaMA pre-training~\citep{touvron2023llama}.
Following~\citet{hoffmann2022training}, the final learning rate \texttt{lr\_min} is set to $1/10$ of the peak learning rate \texttt{lr\_max}. Additionally, 
\begin{itemize}[leftmargin=2em]
    \item {\bf C4 pre-training.}  We follow the setup of \citet{zhao2024galore,chen2024fira,zhu2024apollo}, using a sequence length of 256 and batch size of 512.
    Following the Chinchilla scaling law~\citep{hoffmann2022training}, the total number of training tokens is set to be approximately 20 times the number of model parameters. The training includes 1,000 warm-up steps.
    The grid search for \texttt{lr\_max} is performed over $\{$\texttt{1e-4}, \texttt{2e-4}, \texttt{3e-4}, \texttt{6e-4}, \texttt{1e-3}, \texttt{1.5e-3}$\}$.
    Optimal learning rates for each model are detailed in Tables~\ref{table: dense model config and max lrs} and~\ref{table: moe model config and max lrs}.

    \item {\bf OpenWebText pre-training.} The (max) sequence length is set to 1024, and the batch size is set to 480, following nanoGPT~\citep{Karpathy2022} and~\citet{liu2023sophia}. The total training duration is 50,000 or 100,000 steps, including 1,000 warm-up steps.
    The grid search for \texttt{lr\_max} is performed over $\{$\texttt{2e-4}, \texttt{4e-4}, \texttt{6e-4}, \texttt{8e-4}, \texttt{1e-3}$\}$. Optimal learning rates for each model are detailed in Table~\ref{table: dense model config and max lrs}.
\end{itemize}

{\bf AdamWPower experiments.} 
We adopt $p=1.2$ as the default in all experiments in Section~\ref{subsec: dense results} and~\ref{subsec: moe results}. All other optimizer hyperparameters are kept identical to those used for the AdamW baselines.
Importantly, the power $p=1.2$ proves to be {\bf highly robust}.

\subsection{Experimental details for Section~\ref{subsec: blockwise lr and muon}}

{\bf AdamW with Blockwise LR.}
Following~\citet{wang2025sharpness}, we adopt the same peak \texttt{lr\_max} tuned for AdamW as the \texttt{lr\_max} of AdamW with Blockwise LR. 
For the blockwise lr ratios, we adopt the recommended $r({\rm Embed})=10, r({\rm QK})=8, r({\rm FFN})=6, r({\rm VO})=4$ in~\citet{wang2025sharpness}.

{\bf AdamWPower with Blockwise LR.} 
We still adopt $p=1.2$ in the AdamWPower with Blockwise LR.
All other optimizer hyperparameters are kept identical to those used for the AdamW with Blockwise LR.

{\bf Muon baseline.}
We use the same techniques for Muon as~\citet{liu2025muon}: (1) adding weight decay (2)
adjusting the per-parameter update scale. These techniques allow our Muon experiment to use the identical learning rate as the AdamW baseline without the extra effort of hyper-parameter tuning.

{\bf MuonPower.} 
We still adopt $p=1.2$ in the MuonPower.
All other optimizer hyperparameters are kept identical to those used for the Muon baseline.

\subsection{Experimental details for Section~\ref{subsec: experiment: batch size}}

We conduct experiments using LLaMA (0.2B) on C4 dataset with \texttt{wsd} scheduler.
Unlike the previous experimental settings, here we vary the batch size from the standard 512 up to 8192.

For batch size 512, the tuned \texttt{max\_lr} is \texttt{1e-3} (Table~\ref{table: dense model config and max lrs}).
For larger batch sizes (2048, 4096, 8192), we tune the \texttt{max\_lr} over $\{$\texttt{6e-4, 1e-3, 2e-3, 4e-3, 8e-3}$\}$ for AdamW. We find that \texttt{1e-3} consistently yields the best results across all batch sizes.

For each batch size, we evaluate AdamWPower with multiple values of $p$, and record their validation loss when the optimal validation loss reaches approximately 3.5.

We also conduct vision experiments using ResNet-34 on CIFAR-10 datset with \texttt{cos} scheduler. We tune the \texttt{max\_lr} over $\{$\texttt{6.25e-5, 1.25e-4, 2.5e-4, 5e-4, 1e-3}$\}$. For batch size 32, 64, and 128, the tuned \texttt{max\_lr} is \texttt{1.25e-4}, \texttt{2.5e-4}, \texttt{5e-4}, respectively.

\subsection{Additional experiments with multiple random seeds}

In this subsection, we reproduce a subset of experiments in Figure~\ref{fig: llama on c4, full} with multiple random seeds to assess statistical robustness. Specifically, we rerun the experiments six times with different random seeds and report both mean and standard deviation as shown in Figure~\ref{fig: llama on c4, multiple seed}. The shaded regions in the plots denote the standard deviation, showing the statistical significance of each method. These results confirm that the observed performance differences are consistent and cannot be explained by random seed variability.

\begin{figure}[!htb]
    \centering
    \includegraphics[width=0.3\linewidth]{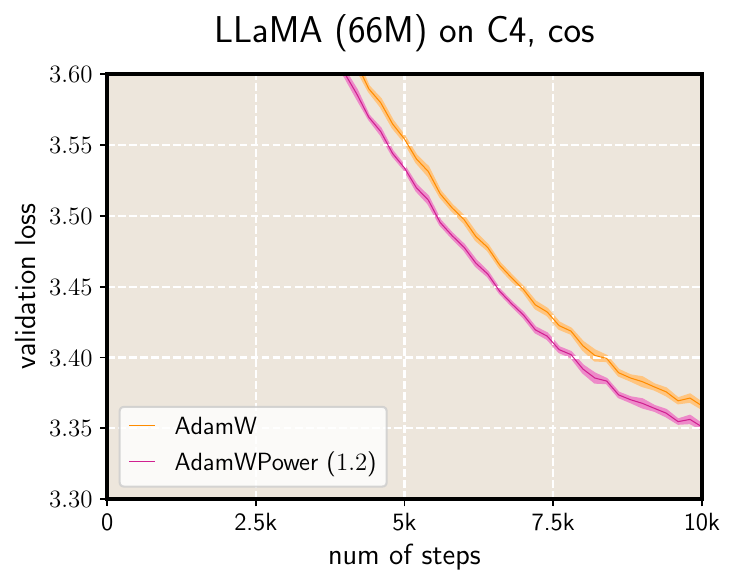}
    \includegraphics[width=0.3\linewidth]{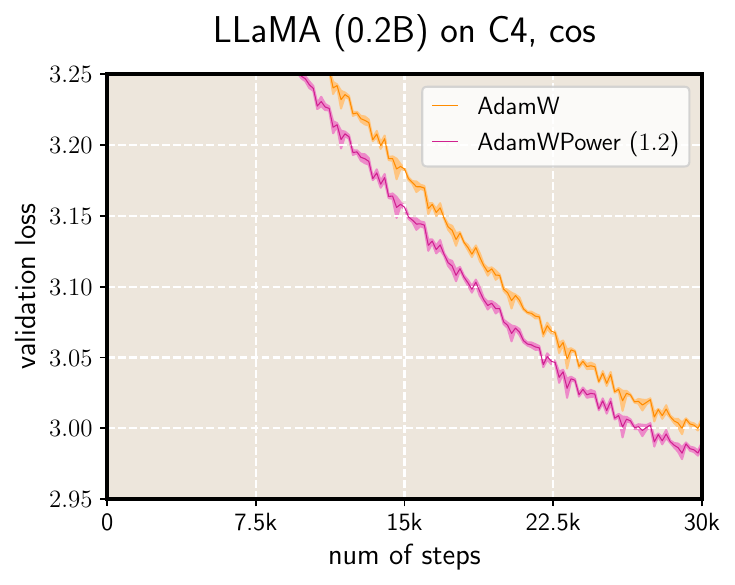}
    \vspace{-.3cm}
    \caption{AdamWPower ($p=1.2$) consistently outperforms AdamW in LLaMA pre-training tasks. The shaded regions in the plots denote the standard deviation.}
    \label{fig: llama on c4, multiple seed}
\end{figure}

\subsection{Robustness to $(\beta_1,\beta_2)$ tuning}

We further examine whether the advantage of AdamWPower persists after tuning the momentum parameters $(\beta_1,\beta_2)$ in AdamW. While the main experiments use the standard setting $(\beta_1,\beta_2)=(0.9,0.95)$, we additionally sweep $(\beta_1,\beta_2)$ for both AdamW and AdamWPower using LLaMA (0.25B) on OpenWebText. Specifically, we consider $(\beta_1,\beta_2)$ configurations following the spirit of \citet{orvieto2026search}:

\[
\begin{aligned}
(\beta_1,\beta_2)\in\{&
(0.9,0.8),\ (0.9,0.9),\ (0.9,0.95),\ (0.9,0.975),\\
&
(0.95,0.9),\ (0.95,0.95),\ (0.95,0.975),\ (0.95,0.9875),\\
&
(0.975,0.95),\ (0.975,0.975),\ (0.975,0.9875),\ (0.975,0.99375)
\}.
\end{aligned}
\]
The best AdamW baseline is obtained with $(\beta_1,\beta_2)=(0.95,0.9875)$, achieving a final validation loss of $2.832831$, whereas AdamWPower with $p=1.2$ achieves its best result with $(\beta_1,\beta_2)=(0.9,0.8)$ and a lower final validation loss of $2.828666$. These results show that AdamWPower continues to outperform the tuned AdamW baseline even after tuning $\beta_1$ and $\beta_2$, suggesting that the improvement is not merely due to a particular choice of momentum parameters.

\section{Proofs in Section~\ref{sec: motivation} and~\ref{subsec: case study}}
\label{appendix: proof of case study}

\subsection{Support for the motivation in Section~\ref{sec: motivation}}

In this section, we provide a detailed justification for the claim in Section~\ref{sec: motivation} that the \textbf{linear transformation} ($\varphi(z)=cz$ with $c\in\bbR$) {\bf fails} to alter the updates of popular optimizers used in LLM pretraining.
Without loss of generality, we analyze the one-dimensional case.

\begin{itemize}[leftmargin=2em]
    \item  {\bf Adaptive optimizers}, including Adagrad, RMSprop, and Adam. 
    These optimizers adjust the learning rate based on a moving average of gradients. 
    In practice, the term $\epsilon$ (used to ensure numerical stability) is typically set to an extremely small value (e.g., \texttt{1e-8, 1e-12}), 
    Consider the update rule of Adam in the limit $\epsilon \to 0$:
    \begin{align*}
        \theta_{t+1}=(1-\lambda\eta_t)\theta_t-\eta_t\frac{\EMA_{\beta_1}(\{g_s\}_1^t)}{\sqrt{\EMA_{\beta_2}(\{g_s^2\}_1^t)}}.
    \end{align*}
    Applying a linear transformation $\varphi(z)=c z$  with $c>0$ does not change the ratio:
    \begin{align*}
    \frac{\EMA_{\beta_1}(\{g_s\}_1^t)}{\sqrt{\EMA_{\beta_2}(\{g_s^2\}_1^t)}}=\frac{\EMA_{\beta_1}(\{\varphi(g_s)\}_1^t)}{\sqrt{\EMA_{\beta_2}(\{\varphi(g_s)^2\}_1^t)}}.
    \end{align*}
    Hence, the dynamics remains unchanged. This argument applies similarly to Adagrad and RMSprop. 

    \item {\bf Sign-based methods}, including
    Sign momentum~\citep{bernstein2018signsgd} and Lion~\citep{chen2024symbolic}.
    These methods operate on the sign of the  moving average gradients. 
    For instance, Signed Momentum (with decoupled weight decay) follows:
    \begin{align*}
    \theta_{t+1}=(1-\lambda\eta_t)\theta_t-\eta_t\sign(\EMA_{\beta}(\{g_s\}_1^t)).
    \end{align*}
    Again, applying a linear transformation $\varphi(z) = c z$ with $c>0$ does not change the sign of the averaged gradient, since:
    \begin{align*}
    \sign(\EMA_{\beta}(\{g_s\}_1^t))=\sign(\EMA_{\beta}(\{\varphi(g_s)\}_1^t)).
    \end{align*}
    Hence, the dynamics remains unchanged. This argument applies similarly to Lion. 
\end{itemize}

In contrast, our proposed (nonlinear) GradPower transformation ($\varphi_p(z)= z^p:=|z|^p\sign(z)$ with $p>0$) \emph{does} alter the updates of both adaptive and sign-based optimizers, when the gradients $g_s$ are not all of the same sign.

\subsection{Proof of Propositions~\ref{prop: low-noise regime} and~\ref{prop: high-noise regime}}


\begin{gather*}
    \bbE[\varphi_p(g)]=\frac{(\mu+\sigma)^{p+1}-|\mu-\sigma|^{p+1}}{2\sigma{(p+1)}},\\
    \bbE[\varphi_p^2(g)]=\frac{(\mu+\sigma)^{2p+1}-|\mu-\sigma|^{2p+1}\sign(\mu-\sigma)}{2\sigma{(2p+1)}}.
\end{gather*}

\underline{The low-noise regime.} ($0\ll\sigma\ll\mu\ll1$)

It is straightforward that
\begin{align*}
    \bbE[\varphi_p(g)]=&\frac{\mu^{p+1}}{2\sigma(p+1)}\left(\left(1+\frac{\sigma}{\mu}\right)^{p+1}-\left(1-\frac{\sigma}{\mu}\right)^{p+1}\right)
    \\=&\frac{\mu^{p+1}}{2\sigma(p+1)}\left(\frac{2(p+1)\sigma}{\mu}+o\left(\frac{\sigma}{\mu}\right)\right)=\mu^p\left(1+o(1)\right);
\end{align*}

\begin{align*}
    \bbE[\varphi_p^2(g)]=&\frac{\mu^{2p+1}}{2\sigma(2p+1)}\left(\left(1+\frac{\sigma}{\mu}\right)^{2p+1}-\left(1-\frac{\sigma}{\mu}\right)^{2p+1}\right)
    \\=&\frac{\mu^{2p+1}}{2\sigma(2p+1)}\left(\frac{2(2p+1)\sigma}{\mu}+o\left(\frac{\sigma}{\mu}\right)\right)=\mu^{2p}\left(1+o(1)\right).
\end{align*}

Therefore, we have
\begin{align*}
    u=\frac{\bbE[\varphi_p(g)]}{\sqrt{\bbE[\varphi_p^2(g)]}+\epsilon}=\frac{\mu^p\left(1+o(1)\right)}{\mu^p\left(1+o(1)\right)+\epsilon}=\frac{1+o(1)}{1+\frac{\epsilon}{\mu^p}}.
\end{align*}

\underline{The high-noise regime.} ($0\ll\mu\ll\sigma\ll1$) 

It is straightforward that
\begin{align*}
    \bbE[\varphi_p(g)]=&\frac{\sigma^{p+1}}{2\sigma(p+1)}\left(\left(1+\frac{\mu}{\sigma}\right)^{p+1}-\left(1-\frac{\mu}{\sigma}\right)^{p+1}\right)
    \\=&
    \frac{\sigma^p}{2(p+1)}\left(\frac{2(p+1)\mu}{\sigma}+o\left(\frac{\mu}{\sigma}\right)\right)
    =\sigma^{p-1}\mu (1+o(1));
\end{align*}

\begin{align*}
    \bbE[\varphi_p^2(g)]=&\frac{\sigma^{2p+1}}{2\sigma(2p+1)}\left(\left(1+\frac{\mu}{\sigma}\right)^{2p+1}+\left(1-\frac{\mu}{\sigma}\right)^{2p+1}\right)
    \\=&
    \frac{\sigma^{2p+1}}{2\sigma(2p+1)}\left(2+o\left(\frac{\mu}{\sigma}\right)\right)=\frac{\sigma^{2p}}{2p+1}(1+o(1)).
\end{align*}

Therefore, we have
\begin{align*}
    u=\frac{\bbE[\varphi_p(g)]}{\sqrt{\bbE[\varphi_p^2(g)]}+\epsilon}
    =\frac{\sigma^{p-1}\mu(1+o(1))}{\frac{\sigma^p}{\sqrt{2p+1}}(1+o(1))+\epsilon}=\frac{\mu}{\sigma}\frac{1+o(1)}{\frac{1+o(1)}{\sqrt{2p+1}}+\frac{\epsilon}{\sigma^p}}=\frac{\mu}{\sigma}\frac{1+o(1)}{\frac{1}{\sqrt{2p+1}}+\frac{\epsilon}{\sigma^p}}.
\end{align*}

To study the monotonicity of $\tilde{u}=\frac{\mu}{\sigma}\frac{1}{\frac{1}{\sqrt{2p+1}}+\frac{\epsilon}{\sigma^p}}$, we only need to study the monotonicity of
\begin{align*}
    \psi(p)=\frac{1}{\sqrt{2p+1}}+\frac{\epsilon}{\sigma^p}.
\end{align*}

It is clear that
\begin{align*}
    \psi'(p)=\frac{\epsilon\log(1/\sigma)}{\sigma^p}-\frac{1}{(2p+1)^{3/2}}.
\end{align*}

Due to $\epsilon\log(1/\sigma)<1$, there exists a $p^\star$, such that $\psi'(p)<0$ for all $0<p<p^\star$; $\psi'(p)>0$ for all $p>p^\star$. Here, $p^\star$ is the solution of the equation:
\begin{align*}
    \frac{\sigma^p}{(2p+1)^{3/2}}=\epsilon\log(1/\sigma)
\end{align*}

Noticing the relationship between $\psi$ and $\tilde{u}$, we have: $\tilde{u}$ increases when $0<p<p^\star$; $\tilde{u}$ decreases when $p>p^\star$.

Now we estimate $p^\star$. Due to $1+x\leq e^x$, we have $(2p+1)^{3/2}\leq(e^{2p})^{3/2}=(e^3)^p$. Then we obtain the two-sides estimate $1\leq(2p+1)^{3/2}\leq(e^3)^p$, implying
\begin{align*}
    \left(\frac{\sigma}{e^3}\right)^p\leq\frac{\sigma^p}{(2p+1)^{3/2}}\leq\sigma^p.
\end{align*}

Therefore, we have the estimate:
\begin{align*}
    \frac{\log(\epsilon\log(1/\sigma))}{\log(\sigma/e^3)}\leq p^\star\leq\frac{\log(\epsilon\log(1/\sigma))}{\log\sigma}
\end{align*}

Noticing $\sigma\ll1$, we obtain:
\begin{align*}
    p^\star=\Theta\left
    (\frac{\log(\epsilon\log(1/\sigma))}{\log\sigma}\right).
\end{align*}



\section{Proofs in Section~\ref{subsec: theory}}
\label{appendix: proof of theory}

Recall that the udpate rule of AdagradPower (with power $p$) follows:
\begin{align*}
    \btheta_{t+1}=&\btheta_t-\eta\bu_t,\\
    \bu_t=&\frac{\varphi_p(\bg_t)}{\sqrt{\bv_t+\epsilon}},\\
    \bv_t=&\sum_{s=1}^t \varphi_p^2(\bg_t).
\end{align*}

In general, our proof is inspired by the main techniques to prove Adagrad used in~\citet{defossez2020simple}. The key difference is to establish a similar estimate of the loss descent for Adamgradpower. This generalize is not trivial, need to use the structure of the high-noise fact.

In the proof, we need an auxiliary sequence, defined as:
\begin{align*}
    \tilde{\bv}_t=\bv_{t-1}+\bbE_t[\varphi_p^2(\bg_t)].
\end{align*}

\subsection{Key Lemmas}

We need two important lemmas in the proof of each Theorem.
The first develops the lower bound of the descent value for the update.

\begin{lemma}[Descent estimate for the update, high-noise regime]\label{lemma: descent lemma}
Under Assumption~\ref{ass: convergence}, for all $t\in\bbN$, and $i\in[d]$ and any $\sigma > 0$, we have:
\begin{align*}
    \bbE_{t}\left[\nabla_i\cL(\btheta) u_{t,i}\right]
    =\bbE_{t}\left[\frac{\nabla_i\cL(\btheta)\varphi_p(g_{t,i})}{\sqrt{v_{t,i}+\epsilon}}\right]\geq
    \frac{\bbE_{t}\left[\nabla_i\cL(\btheta)\varphi_p(g_{t,i})\right]}{\sqrt{\tilde{v}_{t,i}+\epsilon}}
    -\frac{\sigma}{2}\frac{|\nabla_i\cL(\btheta)|^2}{\sqrt{\tilde{v}_{t,i}+\epsilon}}-\frac{2R^p}{\sigma}\bbE\left[\frac{\varphi_p^2(g_{t,i})}{v_{t,i}+\epsilon}\right].
\end{align*}
\end{lemma}

\begin{proof}[Proof of Lemma~\ref{lemma: descent lemma}]\ \\
Let $t\in\bbN$ and $i\in[d]$.
For simplicity, we use the following notations in the proof:
\begin{align*}
    G=\nabla_i\cL(\btheta),\ g=g_{t,i},\ v=v_{t,i},\ \tilde{v}=\tilde{v}_{t,i}.
\end{align*}

First, we decouple the descent quantity as:
\begin{equation}\label{equ: proof of lemma, decouple}
    \bbE_t\left[\frac{G\varphi_p(g)}{\sqrt{v+\epsilon}}\right]=
    \bbE_t\left[\frac{G\varphi_p(g)}{\sqrt{\tilde{v}+\epsilon}}\right]+\bbE_t\Bigg[\underbrace{G\varphi_p(g)\left(\frac{1}{\sqrt{v+\epsilon}}-\frac{1}{\sqrt{\tilde{v}+\epsilon}}\right)}_{I}\Bigg]
\end{equation}

Then we bound the term $I$ in the RHS of Equation~\eqref{equ: proof of lemma, decouple}:
\begin{align*}
    |I|=&|G\varphi_p(g)|\frac{|\tilde{v}-v|}{\sqrt{v+\epsilon}\sqrt{\tilde{v}+\epsilon}(\sqrt{v+\epsilon}+\sqrt{\tilde{v}+\epsilon})}
    \\=&
    |G\varphi_p(g)|\frac{|\bbE_t[\varphi_p^2(g)]-\varphi_p^2(g)|}{\sqrt{v+\epsilon}\sqrt{\tilde{v}+\epsilon}(\sqrt{v+\epsilon}+\sqrt{\tilde{v}+\epsilon})}
    \\\leq&
    |G\varphi_p(g)|\frac{\bbE_t[\varphi_p^2(g)]+\varphi_p^2(g)}{\sqrt{v+\epsilon}\sqrt{\tilde{v}+\epsilon}(\sqrt{v+\epsilon}+\sqrt{\tilde{v}+\epsilon})}
    \\\leq&
    \underbrace{|G\varphi_p(g)|\frac{\bbE_t[\varphi_p^2(g)]}{\sqrt{v+\epsilon}(\tilde{v}+\epsilon)}}_{I_1}+\underbrace{|G\varphi_p(g)|\frac{\varphi_p^2(g)}{(v+\epsilon)\sqrt{\tilde{v}+\epsilon}}}_{I_2}.
\end{align*}

Consequently, we will estimate $I_1$ and $I_2$ by the inequality 
\begin{align*}
    |xy|\leq\frac{\lambda x^2}{2}+\frac{y^2}{2\lambda}.
\end{align*}

For $I_1$, by taking
\begin{align*}
    |x|=\frac{|G|}{\sqrt{\tilde{v}+\epsilon}},\  
    |y|=\frac{|\varphi_p(g)|\bbE_t[\varphi_p^2(g)]}{\sqrt{v+\epsilon}\sqrt{    \tilde{v}+\epsilon}},\ 
    \lambda=\frac{\sigma \sqrt{\tilde{v}+\epsilon}}{2},
\end{align*}

we obtain
\begin{align*}
    I_1\leq&\frac{\sigma}{4}\frac{|G|^2}{\sqrt{\tilde{v}+\epsilon}}+\frac{1}{\sigma}\frac{(\varphi_p^2(g)(\bbE_t[\varphi_p^2(g)])^2}{(v+\epsilon)(\tilde{v}+\epsilon)^{3/2}},
    \\
    \bbE_t[I_1]\leq&
    \frac{\sigma}{4}\frac{|G|^2}{\sqrt{\tilde{v}+\epsilon}}+\frac{1}{\sigma}\frac{(\bbE_t[\varphi_p^2(g)])^2}{(\tilde{v}+\epsilon)^{3/2}}\bbE_t\left[\frac{\varphi_p^2(g)}{v+\epsilon}\right].
\end{align*}

Given that $\sqrt{\bbE_t[\varphi_p^2(g)]}\leq\sqrt{\tilde{v}+\epsilon}$ and $\sqrt{\bbE_{t}[\varphi_p^2(g)]}\leq R^p$, we can simplify the above estimate as:
\begin{align*}
    \bbE_t[I_1]\leq
    \frac{\sigma}{4}\frac{|G|^2}{\sqrt{\tilde{v}+\epsilon}}+\frac{R^p}{\sigma}
    \bbE_t\left[\frac{\varphi_p^2(g)}{v+\epsilon}\right].
\end{align*}

For $I_2$, by taking
\begin{align*}
    |x|=\frac{|G|}{\sqrt{\tilde{v}+\epsilon}},\  
    |y|=\frac{|\varphi_p(g)|\varphi_p^2(g)}{v+\epsilon},\ 
    \lambda=\frac{\sigma\varphi_p^2(g)}{2\bbE_t[\varphi_p^2(g)]},
\end{align*}

we obtain
\begin{align*}
    I_2\leq\frac{\sigma}{4}\frac{\varphi_p^2(g)}{\bbE_t[\varphi_p^2(g)]}\frac{|G|^2}{\sqrt{\tilde{v}+\epsilon}}+\frac{1}{\sigma}\frac{\bbE_t[\varphi_p^2(g)]}{\sqrt{\tilde{v}+\epsilon}}\frac{\varphi_p^4(g)}{(v+\epsilon)^2}
\end{align*}

Given that $\varphi_p^2(g) \leq v+\epsilon$, we can simplify the above estimate as:
\begin{align*}
    I_2\leq\frac{\sigma}{4}\frac{\varphi_p^2(g)}{\bbE_t[\varphi_p^2(g)]}\frac{|G|^2}{\sqrt{\tilde{v}+\epsilon}}+\frac{1}{\sigma}\frac{\bbE_t[\varphi_p^2(g)]}{\sqrt{\tilde{v}+\epsilon}}\frac{\varphi_p^2(g)}{v+\epsilon}.
\end{align*}

Using $\sqrt{\bbE_t[\varphi_p^2(g)]}\leq\sqrt{\tilde{v}+\epsilon}$, $\sqrt{\bbE_{t}[\varphi_p^2(g)]}\leq R^p$, and taking the conditional expectation, we obtain:
\begin{align*}
    \bbE_t[I_2]\leq
    \frac{\sigma}{4}\frac{|G|^2}{\sqrt{\tilde{v}+\epsilon}}+\frac{R^p}{\sigma}\bbE\left[\frac{\varphi_p^2(g)}{v+\epsilon}\right].
\end{align*}

Consequently, combing the two estimates of $I_1$ and $I_2$, we obtain:
\begin{align*}
    \bbE_t[|I|]\leq\bbE_t[I_1]+\bbE_t[I_2]\leq
    \frac{\sigma}{2}\frac{|G|^2}{\sqrt{\tilde{v}+\epsilon}}+\frac{2R^p}{\sigma}\bbE\left[\frac{\varphi_p^2(g)}{v+\epsilon}\right].
\end{align*}

Putting the above estimate into Equation~\eqref{equ: proof of lemma, decouple}, we obtain the lower bound:
\begin{align*}
    \bbE_t\left[\frac{G\varphi_p(g)}{\sqrt{v+\epsilon}}\right]=
    &\bbE_t\left[\frac{G\varphi_p(g)}{\sqrt{\tilde{v}+\epsilon}}\right]+\bbE_t\left[I\right]\geq\bbE_t\left[\frac{G\varphi_p(g)}{\sqrt{\tilde{v}+\epsilon}}\right]-\bbE_t[|I|]
    \\\geq&
    \frac{\bbE_t\left[ G\varphi_p(g)\right]}{\sqrt{\tilde{v}+\epsilon}} - \frac{\sigma}{2}\frac{|G|^2}{\sqrt{\tilde{v}+\epsilon}}-\frac{2R^p}{\sigma}\bbE_t\left[\frac{\varphi_p^2(g)}{v+\epsilon}\right].
\end{align*}
    
\end{proof}

The second lemma estimate the sum of the updates in adaptive methods.

\begin{lemma}[Lemma 5.2 in~\citet{defossez2020simple}]\label{lemma: sum of the updates}
    Let $\{a_t\}_{t\in\bbN}$ be a non-negative sequence, $\epsilon>0$. Then for all $T\in\bbN$, we have:
    \begin{align*}
        \sum_{t=1}^T\frac{a_t}{\epsilon+\sum_{s=1}^t a_s}\leq
        \log\left(1+\frac{1}{\epsilon}\sum_{t=1}^T a_t\right).
    \end{align*}
\end{lemma}

\subsection{Proof of Theorem~\ref{thm: adagradpower, low-noise}}

With the help of the above Lemma~\ref{lemma: descent lemma} and~\ref{lemma: sum of the updates}, we can prove Theorem~\ref{thm: adagradpower, low-noise}.

\begin{proof}[Proof of Theorem~\ref{thm: adagradpower, low-noise}]\ \\
Due to the $H$-smoothness, we have the quadratic upper bound: 
\begin{align*}
    \cL(\btheta_{t+1})\leq\cL(\btheta_{t})-\eta\<\nabla\cL(\btheta_t),\bu_t\>+\frac{\eta^2H}{2}\left\|\bu_t\right\|_2^2.
\end{align*}

Taking the expectation at $t$, we have:
\begin{align*}
\bbE_t\left[\cL(\btheta_{t+1})\right]
\leq&\cL(\btheta_{t})-\eta\bbE_t\left[\<\nabla\cL(\btheta_t),\bu_t\>\right]+\frac{\eta^2H}{2}\bbE_t\left[\left\|\bu_t\right\|_2^2\right]
\\=&
\cL(\btheta_{t})-\eta\sum_{i=1}^d\bbE_t\left[\nabla_i\cL(\btheta_t)u_{t,i}\right]+\sum_{i=1}^d\frac{\eta^2H}{2}\bbE_t\left[u_{t,i}^2\right].
\end{align*}

Combine Lemma~\ref{lemma: descent lemma} with $\sigma =c$ and Assumption~\ref{ass: low-noise regime}, we get

\begin{align*}
    \bbE_{t}\left[\nabla_i\cL(\btheta) u_{t,i}\right]
    =\bbE_{t}\left[\frac{\nabla_i\cL(\btheta)\varphi_p(g_{t,i})}{\sqrt{v_{t,i}+\epsilon}}\right] & \geq
    c\frac{|\nabla_i\cL(\btheta)|^{p+1}}{\sqrt{\tilde{v}_{t,i}+\epsilon}} - \frac{c}{2}\frac{|\nabla_i\cL(\btheta)|^2}{\sqrt{\tilde{v}_{t,i}+\epsilon}}-\frac{2R^p}{c}\bbE\left[\frac{\varphi_p^2(g_{t,i})}{v_{t,i}+\epsilon}\right] \\
    & \geq
    \frac{c}{2}\frac{|\nabla_i\cL(\btheta)|^{p+1}}{\sqrt{\tilde{v}_{t,i}+\epsilon}}-\frac{2R^p}{c}\bbE\left[\frac{\varphi_p^2(g_{t,i})}{v_{t,i}+\epsilon}\right].
\end{align*}

Where last inequality comes from $R < 1 $. Using it for each dimension, we have:
\begin{align*}
\bbE_t\left[\cL(\btheta_{t+1})\right]\leq&\cL(\btheta_{t})-\frac{\eta c}{2}\frac{|\nabla_i\cL(\btheta_t)|^{p+1}}{\sqrt{\tilde{v}_{t,i}+\epsilon}}+\frac{2\eta R^p}{c}\bbE\left[\frac{\varphi_p^2(g_{t,i})}{v_{t,i}+\epsilon}\right]
+\sum_{i=1}^d\frac{\eta^2H}{2}\bbE_t\left[u_{t,i}^2\right]
\\=&
\cL(\btheta_{t})-\sum_{i=1}^d\frac{\eta c}{2}\frac{|\nabla_i\cL(\btheta_t)|^{p+1}}{\sqrt{\tilde{v}_{t,i}+\epsilon}}+\sum_{i=1}^d\left(\frac{2\eta R^p}{c}+\frac{\eta^2 H}{2}\right)\bbE_t\left[\frac{\varphi_p^2(g_{t,i})}{v_{t,i}+\epsilon}\right].
\end{align*}

Noticing $\sqrt{\tilde{v}_{t_i}+\epsilon}\leq R^p\sqrt{t}$, we further have:
\begin{align*}
\bbE_t\left[\cL(\btheta_{t+1})\right]\leq
\cL(\btheta_{t})-\frac{\eta c}{2}\frac{\|\nabla\cL(\btheta_t)\|_{p+1}^{p+1}}{R^p\sqrt{t}}+\sum_{i=1}^d\left(\frac{2\eta R^p}{c}+\frac{\eta^2 H}{2}\right)\bbE_t\left[\frac{\varphi_p^2(g_{t,i})}{v_{t,i}+\epsilon}\right].
\end{align*}

Summing the previous inequality for all $0\leq t\leq T-1$, taking the complete expectation, and using $\sqrt{t}\leq\sqrt{T}$, we have:
\begin{align*}
    \bbE\left[\cL(\btheta_{t})\right]
    \leq
    \cL(\btheta_{0})-\frac{\eta c\sum_{t=1}^T\|\nabla\cL(\btheta_t)\|_{p+1}^{p+1}}{2\eta R^p\sqrt{T}}+\sum_{i=1}^d\left(\frac{2R^p}{c}+\frac{\eta^2 H}{2}\right)\bbE\left[\sum_{t=1}^T\frac{\varphi_p^2(g_{t,i})}{v_{t,i}+\epsilon}\right].
\end{align*}

Then for each dimension, using Lemma~\ref{lemma: sum of the updates} for the sequence $\{(g_{t,i}^p)^2\}_{1\leq t\leq T}$, we obtain:
\begin{align*}
    &\bbE\left[\cL(\btheta_{t})\right]
    \\\leq&
    \cL(\btheta_{0})-\frac{\eta c \sum_{t=1}^T\bbE\|\nabla\cL(\btheta_t)\|_{p+1}^{p+1}}{2 R^p\sqrt{T}}+\left(\frac{2\eta R^p}{c}+\frac{\eta^2 H}{2}\right)d\ \bbE\left[\log\left(1+\frac{1}{\epsilon}\sum_{t=1}^T \varphi_p^2(g_{t,i})\right)\right]
    \\\leq&
    \cL(\btheta_{0})-\frac{\eta c\sum_{t=1}^T\bbE\|\nabla\cL(\btheta_t)\|_{p+1}^{p+1}}{2R^p\sqrt{T}}+\left(\frac{2\eta R^p}{c}+\frac{\eta^2 H}{2}\right)d\log\left(1+\frac{R^{2p}}{\epsilon}T\right).
\end{align*}
This implies:
\begin{align*}
    &\bbE \min_{1\leq t\leq T}\|\nabla\cL(\btheta_t)\|_{p+1}^{p+1} \leq \frac{1}{T}\sum_{t=1}^T\bbE\|\nabla\cL(\btheta_t)\|_{p+1}^{p+1}
    \\\leq&\frac{2R^p}{c \sqrt{T}}\left(\frac{\cL(\btheta_0)-\cL^\star}{\eta} +\left(\frac{2R^p}{c}+\frac{\eta H}{2}\right)d\log\left(1+\frac{R^{2p}}{\epsilon}T\right)\right)
    =\cO\left(\frac{\log T}{\sqrt{T}}\right).
\end{align*}
Hence
\begin{align*}
    \bbE\min_{1\leq t\leq T} \|\nabla\cL(\btheta_t)\|_{p+1}^{p+1} \leq \cO\left(\frac{\log T}{\sqrt{T}}\right).
\end{align*}

\end{proof}

\subsection{Proof of Theorem~\ref{thm: adagradpower, high-noise}}

With the help of the above Lemma~\ref{lemma: descent lemma} and~\ref{lemma: sum of the updates}, we can prove Theorem~\ref{thm: adagradpower, high-noise}.

\begin{proof}[Proof of Theorem~\ref{thm: adagradpower, high-noise}]\ \\
Due to the $H$-smoothness, we have the quadratic upper bound: 
\begin{align*}
    \cL(\btheta_{t+1})\leq\cL(\btheta_{t})-\eta\<\nabla\cL(\btheta_t),\bu_t\>+\frac{\eta^2H}{2}\left\|\bu_t\right\|_2^2.
\end{align*}

Taking the expectation at $t$, we have:
\begin{align*}
\bbE_t\left[\cL(\btheta_{t+1})\right]
\leq&\cL(\btheta_{t})-\eta\bbE_t\left[\<\nabla\cL(\btheta_t),\bu_t\>\right]+\frac{\eta^2H}{2}\bbE_t\left[\left\|\bu_t\right\|_2^2\right]
\\=&
\cL(\btheta_{t})-\eta\sum_{i=1}^d\bbE_t\left[\nabla_i\cL(\btheta_t)u_{t,i}\right]+\sum_{i=1}^d\frac{\eta^2H}{2}\bbE_t\left[u_{t,i}^2\right].
\end{align*}

Combine Lemma~\ref{lemma: descent lemma} with Assumption~\ref{ass: high-noise regime}, we get

\begin{align*}
    \bbE_{t}\left[\nabla_i\cL(\btheta) u_{t,i}\right]
    =\bbE_{t}\left[\frac{\nabla_i\cL(\btheta)\varphi_p(g_{t,i})}{\sqrt{v_{t,i}+\epsilon}}\right]\geq
    \frac{\sigma}{2}\frac{|\nabla_i\cL(\btheta)|^2}{\sqrt{\tilde{v}_{t,i}+\epsilon}}-\frac{2R^p}{\sigma}\bbE\left[\frac{\varphi_p^2(g_{t,i})}{v_{t,i}+\epsilon}\right].
\end{align*}

Using it for each dimension, we have:
\begin{align*}
\bbE_t\left[\cL(\btheta_{t+1})\right]\leq&\cL(\btheta_{t})-\frac{\eta\sigma}{2}\frac{|\nabla_i\cL(\btheta_t)|^2}{\sqrt{\tilde{v}_{t,i}+\epsilon}}+\frac{2\eta R^p}{\sigma}\bbE\left[\frac{\varphi_p^2(g_{t,i})}{v_{t,i}+\epsilon}\right]
+\sum_{i=1}^d\frac{\eta^2H}{2}\bbE_t\left[u_{t,i}^2\right]
\\=&
\cL(\btheta_{t})-\sum_{i=1}^d\frac{\eta\sigma}{2}\frac{|\nabla_i\cL(\btheta_t)|^2}{\sqrt{\tilde{v}_{t,i}+\epsilon}}+\sum_{i=1}^d\left(\frac{2\eta R^p}{\sigma}+\frac{\eta^2 H}{2}\right)\bbE_t\left[\frac{\varphi_p^2(g_{t,i})}{v_{t,i}+\epsilon}\right].
\end{align*}

Noticing $\sqrt{\tilde{v}_{t_i}+\epsilon}\leq R^p\sqrt{t}$, we further have:
\begin{align*}
\bbE_t\left[\cL(\btheta_{t+1})\right]\leq
\cL(\btheta_{t})-\frac{\eta \sigma}{2}\frac{\|\nabla\cL(\btheta_t)\|_2^2}{R^p\sqrt{t}}+\sum_{i=1}^d\left(\frac{2\eta R^p}{\sigma}+\frac{\eta^2 H}{2}\right)\bbE_t\left[\frac{\varphi_p^2(g_{t,i})}{v_{t,i}+\epsilon}\right].
\end{align*}

Summing the previous inequality for all $0\leq t\leq T-1$, taking the complete expectation, and using $\sqrt{t}\leq\sqrt{T}$, we have:
\begin{align*}
    \bbE\left[\cL(\btheta_{t})\right]
    \leq
    \cL(\btheta_{0})-\frac{\eta \sigma\sum_{t=1}^T\|\nabla\cL(\btheta_t)\|_2^2}{2\eta R^p\sqrt{T}}+\sum_{i=1}^d\left(\frac{2R^p}{\sigma}+\frac{\eta^2 H}{2}\right)\bbE\left[\sum_{t=1}^T\frac{\varphi_p^2(g_{t,i})}{v_{t,i}+\epsilon}\right].
\end{align*}

Then for each dimension, using Lemma~\ref{lemma: sum of the updates} for the sequence $\{(g_{t,i}^p)^2\}_{1\leq t\leq T}$, we obtain:
\begin{align*}
    &\bbE\left[\cL(\btheta_{t})\right]
    \\\leq&
    \cL(\btheta_{0})-\frac{\eta \sigma\sum_{t=1}^T\bbE\|\nabla\cL(\btheta_t)\|_2^2}{2 R^p\sqrt{T}}+\left(\frac{2\eta R^p}{\sigma}+\frac{\eta^2 H}{2}\right)d\ \bbE\left[\log\left(1+\frac{1}{\epsilon}\sum_{t=1}^T \varphi_p^2(g_{t,i})\right)\right]
    \\\leq&
    \cL(\btheta_{0})-\frac{\eta\sigma\sum_{t=1}^T\bbE\|\nabla\cL(\btheta_t)\|_2^2}{2R^p\sqrt{T}}+\left(\frac{2\eta R^p}{\sigma}+\frac{\eta^2 H}{2}\right)d\log\left(1+\frac{R^{2p}}{\epsilon}T\right).
\end{align*}
This implies:
\begin{align*}
    &\bbE\min_{1\leq t\leq T}\|\nabla\cL(\btheta_t)\|_2^2\leq \frac{1}{T}\sum_{t=1}^T\bbE\|\nabla\cL(\btheta_t)\|_2^2
    \\\leq&\frac{2R^p}{\sigma\sqrt{T}}\left(\frac{\cL(\btheta_0)-\cL^\star}{\eta} +\left(\frac{2R^p}{\sigma}+\frac{\eta H}{2}\right)d\log\left(1+\frac{R^{2p}}{\epsilon}T\right)\right) \\
    \leq& \frac{R^{p-1}}{\sigma}\frac{2R}{\sqrt{T}} \left( \frac{\cL(\btheta_0)-\cL^\star}{\eta}+ \left(2R +\frac{\eta H}{2}\right)d\log\left(1+\frac{R^2}{\epsilon} T\right)\right) \\
    = &  \frac{R^{p-1}}{\sigma}\Big({\rm R.H.S.} \text{ of~\eqref{equ: adagrad rate}} \Big).
\end{align*}

The last inequality comes from Assumption~\ref{ass: high-noise regime} and $R < 1$.

\end{proof}

\end{document}